\documentclass{article}

\usepackage[final,nonatbib]{neurips_2022}

\usepackage{xcolor}
\usepackage[utf8]{inputenc} % allow utf-8 input
\usepackage[T1]{fontenc}    % use 8-bit T1 fonts
\usepackage[backref=page]{hyperref}       % hyperlinks
\usepackage{url}            % simple URL typesetting
\usepackage{booktabs}       % professional-quality tables
\usepackage{amsfonts}       % blackboard math symbols
\usepackage{nicefrac}       % compact symbols for 1/2, etc.
\usepackage{microtype}      % microtypography
\usepackage{xcolor}         % colors

\usepackage{amsmath,amssymb,mathrsfs,amsthm,mathtools}
\usepackage[capitalise]{cleveref}
\usepackage{algorithm}
\usepackage{algorithmic}
\usepackage{amsmath,graphicx}
\usepackage[american]{babel}
\usepackage{color}
\usepackage{amssymb}
\usepackage{multirow}
\usepackage{amsfonts,mathrsfs}
\usepackage{bm,url}
\usepackage{booktabs,subfigure}
\usepackage{amsthm}
\usepackage{microtype}
\usepackage{multirow}
\usepackage{threeparttable}
\usepackage{booktabs} % for professional tables
\usepackage{float}
\usepackage{wrapfig}
\usepackage{float}
\usepackage{tikz} 
\usepackage{schemata}
\newcommand\AB[2]{\schema{\schemabox{#1}}{\schemabox{#2}}}
\usepackage{enumitem}

\usepackage{selectp}
\definecolor{fhcolor}{rgb}{0.523, 0.235, 0.625}

\DeclareMathOperator*{\argmin}{argmin}

\newcommand{\sind}[3]{{#1}^{#2}_{#3}} 

\newtheorem{definition}{Definition}

\newtheorem{theorem}{Theorem}
\newtheorem{lemma}{Lemma}
\newtheorem{proposition}{Proposition}
\newtheorem{corollary}{Corollary}

\newtheorem{assumption}{Assumption}

\usepackage[framemethod=tikz]{mdframed}

\definecolor{ocre}{RGB}{243,102,25}
\definecolor{mygray}{RGB}{243,243,244}

\newmdenv[
innertopmargin=0pt,
backgroundcolor=mygray,
linecolor=ocre,
innerleftmargin=0pt,
innerrightmargin=0pt,
leftmargin=10pt
]{mymath}

\title{Understanding Deep Neural Function Approximation in Reinforcement Learning via $\epsilon$-Greedy Exploration}

\author{%
	Fanghui Liu\thanks{Correspondence to: Fanghui Liu \texttt{<fanghui.liu@epfl.ch>} and Luca Viano\texttt{<luca.viano@epfl.ch>}.}, Luca Viano, Volkan Cevher \\
	Laboratory for Information and Inference Systems\\
	 \'{E}cole Polytechnique F\'{e}d\'{e}rale de Lausanne (EPFL), Switzerland\\
	 \texttt{\{first\}.\{last\}@epfl.ch} \\
}

\begin{document}
	\maketitle
	\begin{abstract}
This paper provides a theoretical study of deep neural function approximation in reinforcement learning (RL) with the $\epsilon$-greedy exploration under the online setting.
This problem setting is motivated by the successful deep Q-networks (DQN) framework that falls in this regime.
In this work, we provide an initial attempt on theoretical understanding deep RL from the perspective of function class and neural networks architectures (e.g., width and depth) beyond the ``linear'' regime.
To be specific, we focus on the value based algorithm with the $\epsilon$-greedy exploration via deep (and two-layer) neural networks endowed by Besov (and Barron) function spaces, respectively, which aims at approximating an $\alpha$-smooth Q-function in a $d$-dimensional feature space. 
We prove that, with $T$ episodes, scaling the width $m = \widetilde{\mathcal{O}}(T^{\frac{d}{2\alpha + d}})$ and the depth $L=\mathcal{O}(\log T)$ of the neural network for deep RL is sufficient for learning with sublinear regret in Besov spaces. 
Moreover, for a two layer neural network endowed by the Barron space, scaling the width $\Omega(\sqrt{T})$ is sufficient.
To achieve this, the key issue in our analysis is how to estimate the temporal difference error under deep neural function approximation as the $\epsilon$-greedy exploration is not enough to ensure ``optimism''.
Our analysis reformulates the temporal difference error in an $L^2(\mathrm{d}\mu)$-integrable space over a certain averaged measure $\mu$, and transforms it to a generalization problem under the non-iid setting. This might have its own interest in RL theory for better understanding $\epsilon$-greedy exploration in deep RL.

	\end{abstract}
	
\section{Introduction}
\vspace{-0.1cm}

Efficient reinforcement learning (RL) under the large (or even infinite) state space and action space setting is increasingly important and relevant challenge \cite{szepesvari2010algorithms,sutton2018reinforcement,fujimoto2018addressing}. One of the first successful approaches towards this problem is
the deep Q-network (DQN) \cite{mnih2015human,silver2016mastering} framework, which deploys powerful nonlinear function approximation techniques via Deep Neural Networks (DNNs) \cite{lecun2015} to concisely approximate state and action spaces. 
%DQN and its variants utilizes deep neural function approximation and conducts the exploration via the $\epsilon$-greedy policy under the online setting. 
Despite its impressive practical success, there is still a gap between practical uses and theoretical understanding on deep RL with regard to the function class and the employed $\epsilon$-greedy policy.

In the perspective of function class, many theoretical works center around linear function approximation \cite{jin2020provably,wang2020optimism} and linear mixtures \cite{ayoub2020model,zhou2021nearly}.
Existing non-linear function approximation results on RL are largely based on neural tangent kernel (NTK) \cite{jacot2018neural,yang2020function}, Bellman rank \cite{jiang2017contextual,dong2020root}, and Eluder dimension \cite{russo2013eluder,wang2020reinforcement,ishfaq2021randomized}.
Nevertheless, these approaches fail in truly capturing the highly non-linear properties of deep RL. For example, NTK (or lazy training \cite{chizat2019lazy}) essentially works in a ``linear'' regime \cite{lee2019wide,woodworth2020kernel,geiger2020disentangling}, and can not efficiently learn even a single ReLU neuron \cite{bach2017breaking,yehudai2019power,celentano2021minimum} as it requires  $\Omega(\varepsilon^{-d})$ samples to achieve $\varepsilon$ approximation error, where $d$ is the (original) or transformed feature dimension input;
the Bellman rank is normally difficult to be estimated for neural networks as suggested by \cite{huang2021going};
the Eluder dimension is at least in an exponential order \cite{li2021eluder,dong2021provable} even for two-layer neural networks.
The above general function approximation schemes appear difficult to fully demonstrate the success of practical deep RL both theoretically and empirically.

In the perspective of exploration schemes, DQN is directly equipped with the $\epsilon$-greedy policy instead of confidence-bound based scheme that are commonly used in RL theory. 
The $\epsilon$-greedy exploration is theoretically demonstrated to have exponential sample complexity in the worst case \cite{osband2019deep} but is still popular in practical deep RL due to its simple implementation.
In this case, theoretical analyses of $\epsilon$-greedy in deep RL are still required. 
Besides, to ensure a sublinear regret, under the NTK regime, the width of neural networks is required to be $m=\Omega(T^{13})$ \cite{yang2020function}, where $T$ is the number of episodes. This does not match deep RL in practice with small width/depth under large episodes \cite{mnih2015human,hester2018deep}.

To bridge the large theory-practice gap, we study the value iteration algorithm with deep neural function approximation and the $\epsilon$-greedy policy under the online setting, which broadly captures the key features of DQN.
Our analysis framework is based on DNNs (as well as two-layer neural networks) where the target Q function lies in the Besov space \cite{suzuki2019adaptivity} or the Barron space \cite{weinan2021barron}, respectively.
These function classes can fully capture the properties of Q-functions, e.g., smoothness by neural networks. 
Our results demonstrate that the sublinear regret can be achieved for deep neural function approximation under the $\epsilon$-greedy exploration with reasonably finite width and depth in practice.
Besides, the relationship between the problem-dependent smoothness of Q-function and regret bounds is also developed. 
These results could also motivate practitioners to consider different architectures of implementations of deep RL.

\vspace{-0.0cm}
\subsection{Technical challenges and contributions}
\vspace{-0.0cm}

Most previous RL theory results on function approximation in the online setting work with ``optimism in the face of uncertainty'' principle for exploration, leading to a series of upper confidence bound (UCB)-type algorithms to ensure the temporal difference (TD) error smaller than zero. 

Conceptually, optimism is sometimes too aggressive and UCB-style algorithms can suffer exponential sample complexity even for nonlinear bandits \cite{dong2021provable}. 
Technically, UCB-type algorithms in linear/kernel function approximation \cite{jin2020provably,yang2020function,yang2020reinforcement} depend on a known feature mapping or the NTK kernel, which appears invalid for deep neural function approximation beyond the ``linear'' regime. This is because, the used confidence ellipsoid and elliptical potential lemma are not applicable for data-dependent feature mapping of DNNs.
To avoid explicitly designing a bonus function, Thompson sampling \cite{russo2018tutorial,agrawal2012analysis} appears promising in a Bayesian perspective by using randomized (i.e., perturbed) versions of the estimated model or value function \cite{zanette2020frequentist}. 
Nevertheless, the bonus function is still implicitly included in confidence estimate of perturbations.

In this work, we center around deep neural function approximation with the $\epsilon$-greedy exploration.
Since this exploration scheme is not enough to ensure the TD error smaller than zero, the technical challenge in our analysis is how to estimate it to ensure the sublinear regret. 
In our proof framework, by a measure transform, the TD error is analysed in an $L^2(\mathrm{d} \bar{\mu})$-integrable space, where $\bar{\mu}$ is the averaged measure wrt a mini-batch of historical state-action pairs.
To break the dependence between the episodes for neural networks training, we utilize the \emph{experience replay} scheme \cite{lin1992self} from DQN, and then transform the TD error estimation to generalization error under the independent but non-identically distributed data setting and approximation error in the respective function spaces. Note that in practice, experience replay makes observations to be (nearly) iid, but our analysis only requires the independence of observations, that is weaker than iid.
Such generalization problem can be addressed by uniform convergence via (local) Rademacher complexity of the Besov/Barron spaces under the averaged measure.
This considered function spaces in this work is more general than H\"{o}lder spaces used in offline RL \cite{fan2020theoretical}. 

Our results show that (\textit{i}) the problem-dependent smoothness of Q-function affects the efficiency of learning with deep RL, which can be improved by increasing the model capacity (width and depth). We use $\alpha$ as a parameter indicating the smoothness degree of Q-function. A larger $\alpha$ indicates smoother functions, easier RL tasks, and smaller exploration times, which coincides with our theory. 
(\textit{ii}) for deep neural networks under the Besov space, the width $m = \widetilde{\mathcal{O}}(T^{\frac{d}{2\alpha + d}})$ and the depth $L=\widetilde{\mathcal{O}}(1)$ are enough for sublinear regret under the $\epsilon$-greedy policy, where $\widetilde{\mathcal{O}}(\cdot)$ omits the $\log$ terms.
(\textit{iii}) for two-layer neural networks under the Barron space, the width $m = \Omega(\sqrt{T})$ suffices to ensure sublinear regret.
Furthermore, our regret bounds can be independent of the feature dimension, supporting the premise of practical, high-dimensional data in RL.

\vspace{-0.2cm}
\subsection{Related work}
\vspace{-0.2cm}

Recent work on neural network function approximation beyond NTK (or the Eluder dimension) mainly restrict on the generative setting \cite{huang2021going,long2021L2} by assuming a simulator in which the agent can require any state and action, and the offline setting \cite{fan2020theoretical,nguyen2021sample}. In sequel, we review RL with function approximation under the online setting that DQN falls into this regime. 
We also mention that, theoretical understanding of DQN can be conducted by from the perspective of neural fitted Q-iteration algorithm \cite{riedmiller2005neural,fan2020theoretical,xu2020finite}, and Q learning \cite{jin2018q} in the perspective of understanding the target network \cite{zanette2022stabilizing} and experience replay \cite{carvalho2020new,agarwal2021online,szlak2021convergence} with linear function approximation.
Note that, for notational consistency with previous work, in this subsection, $T$ denotes the total number of steps (i.e., interactions with the environment) instead of the number of episodes in our paper.

{\bf RL with linear/kernel function approximation:} RL with linear function approximation achieves a sublinear regret bound with $\widetilde{\mathcal{O}}(\sqrt{d^3H^3 T})$ under a low-rank MDP in a model-free setting \cite{jin2020provably} and $\widetilde{\mathcal{O}}(dH^2\sqrt{T})$ in a model-based setting \cite{yang2020reinforcement}, where $H$ is the length of each episode.
The regret can be improved to $\widetilde{\mathcal{O}}(dH\sqrt{T})$ under a low inherent Bellman error by assuming a global planning oracle \cite{zanette2020learning} or under a Bernstein-type exploration bonus and controlling extra uniform convergence cost \cite{hu2022nearly}. 
This nearly optimal regret can be also achieved under the linear mixtures setting \cite{zhou2021nearly}.
In the kernel regime, the regret can be achieved with $\widetilde{\mathcal{O}}({\delta}_{\mathcal{F}}\sqrt{H^3T})$ \cite{yang2020reinforcement,yang2020function}, where ${\delta}_{\mathcal{F}}$ is the intrinsic complexity (e.g., effective dimension) of the function class RKHS $\mathcal{F}$. The above bounds are based on confidence ellipsoid to quantify the uncertainty in an explicit bonus function by feature mapping/kernel function; while Thompson sampling \cite{agrawal2012analysis,russo2018tutorial} utilizes an implicit bonus function in probability estimation on uncertainty quantification, which leads to an $\widetilde{\mathcal{O}}(d^2H^2\sqrt{T})$ \cite{zanette2020frequentist} regret in linear function approximation.

{\bf RL with general function approximation:}  One prototypical scheme uses the Eluder dimension \cite{russo2013eluder}, which measures the degree of dependence among action rewards, resulting in an $\widetilde{\mathcal{O}}({\tt poly}({\delta}_{\mathcal{F}} H) \sqrt{T})$ regret \cite{wang2020reinforcement,ishfaq2021randomized}, where the complexity ${\delta}_{\mathcal{F}}$ depends on the Eluder dimension. 
Using this metric, the sublinear regret under the $\epsilon$-greedy exploration can be achieved by \cite{dann2022guarantees}.
Besides, the low Bellman rank assumption \cite{jiang2017contextual}, where the Bellman error ``matrix'' admits a low-rank factorization, can be also used general function approximation \cite{dong2020root} by measuring the error of the function class under the Bellman operator. 
Combining Bellman rank and Eluder dimension results in a new metric, Bellman Eluder dimension \cite{jin2021bellman}, achieving $\widetilde{\mathcal{O}}(H\sqrt{{\delta}_{\mathcal{F}} T})$-regret, where ${\delta}_{\mathcal{F}}$ depends on this metric.

Overall, the above metrics are difficult to the nonlinear spaces of DNNs beyond ``linear'' regime that concern us. 

\section{Background and preliminaries}

In this section, we introduce the necessary background and
definitions with respect to online reinforcement learning based on episodic Markov decision processes (MDPs) and function spaces of deep (and two-layer) ReLU neural networks. 

{\bf Notation:} We denote by $a(n) \lesssim b(n)$: there exists a positive constant $c$ independent of $n$ such that $a(n) \leqslant c b(n)$; $a(n) \asymp b(n)$: there exists two positive constant $c_1$ and $c_2$ independent of $n$ such that $c_1 b(n) \leqslant a(n) \leqslant c_2 b(n)$. 
We use the shorthand $ [n]:= \{1,2,\dots, n \}$ for some positive $n$ and $\lceil x \rceil$ denotes the smallest integer exceeding $x$. 
Let $\mathcal{X} =[0,1]^d$ be a domain of the functions, we denote the $L^p$-integrable space by $L^p(\mathcal{X})$ endowed by the norm $\|f\|_{L^{p}(\mathcal{X})} = \big( \int_{\mathcal{X} } |f(\bm x)|^{p} \mathrm{d} \bm x \big)^{1/p} $, and the $\mu$-integrable $L^p$ space by $L^p(\mathrm{d} \mu)$  for a probability measure $\mu$ on $\mathcal{X}$ and the norm is given by $\|f\|_{L^{p}(\mathrm{d}\mu)} = \big( \int_{\mathcal{X} } |f(\bm x)|^{p} \mathrm{d} \mu \big)^{1/p} $.

\subsection{Episodic Markov decision processes} 

A (finite-horizon) episodic MDPs is denoted as $\text{MDP}(\mathcal{S}, \mathcal{A}, H, \mathbb{P}, r)$,
where $\mathcal{S}$ is the state space with possibly infinite states; $\mathcal{A}$ is the finite action space; 
$H$ is the number of steps in each episode; $\mathbb{P} := \{ \mathbb{P}_h \}_{h=1}^H$ is the Markov transition kernel with the transition probability  $\mathbb{P}_h ( \cdot | s, a) $ on action $a$ taken at state $s \in \mathcal{S}$ in the $h$-th step; the reward functions $r := \{ r_h \}_{h=1}^H $ are assumed to be deterministic. %$r_h: \mathcal{S} \times \mathcal{A} \rightarrow [0,1]$ at $h$-th step.
For notational simplicity, denote $\mathcal{X}= \mathcal{S} \times \mathcal{A}$ and $\bm x=(s,a)$, we assume $\mathcal{X}=[0,1]^d$ as a compact space of $\mathbb{R}^d$ and $r_h: \mathcal{S} \times \mathcal{A} \rightarrow [0,1]$ at $h$-th step. %, where the dimension 

A non-stationary policy $\pi$ is a collection of $H$ functions $\pi:= \{ \pi_h: \mathcal{S} \rightarrow {\mathcal{A}} \}_{h=1}^H$.
%The performance of the agent is captured by the \emph{value function}. 
Given a policy $\pi$, the (state) value function $V_h^{\pi}\colon \mathcal{S} \to [0,H]$ is defined as the expected cumulative reward of the MDP starting from step $h \in [H]$, i.e.,  $	V_h^\pi(s) =  \mathbb{E}_{\pi} \big[\sum_{h' = h}^H r_{h'}(s_{h'},  a_{h'} )  
\big|  s_h = s \big], \forall s\in \mathcal{S}, h \in [H]$
where $\mathbb{E}_{\pi}[\cdot]$ denotes the expectation with respect to the 
randomness of the trajectory $\{(s_h,a_h)\}_{h=1}^H$ obtained by
the  policy $\pi $. Likewise, the action-value function  
$Q_h^\pi:\mathcal{S} \times \mathcal{A} \to [0,H]$ is defined as $	Q^{\pi}_h(s,a) =\mathbb{E}_{\pi} \big[ \sum_{h'=h}^H r_{h'} (s_{h'},a_{h'} ) 
\,\big|\, s_h=s,\,a_h=a  \big]$.

Moreover, since the action space and episode length are both finite, there always exists an optimal policy $\pi^\star$ \cite{puterman2014markov} such that $V^{\star}_{h}(s) = \sup_{\pi} V_h^\pi(s)$ 	for all $s\in \mathcal{S}$ and $h\in [H]$. To simplify the notation,  denote
$	(\mathbb{P}_h V  ) (s, a) := \mathbb{E}_{s' \sim \mathbb{P}_h(\cdot | s, a)} [V (s')]
$ and the Bellman operator $ (\mathbb{T}_h V)(s, a) = r_h(s,a) + (\mathbb{P}_h V  ) (s, a)$
for any measurable function $V \colon \mathcal{S} \rightarrow [ 0, H] $.
Using this notation, the Bellman equation associated with a policy $\pi$ can be formulated as
\begin{align} \label{eq:bellman} 
	\sind{Q}{\pi}{h}(s, a) = (\mathbb{T}_h \sind{V}{\pi}{h+1})(s, a)  , \qquad   V_{h}^\pi (s) = \langle Q_h^{\pi} (s, \cdot ), \pi_{h}(\cdot | s)  \rangle_{\mathcal{A}} ,     
	\qquad \sind{V}{\pi}{H+1}(s) = 0 \,.
\end{align}
Similarly, the Bellman optimality equation is given by 
\begin{align}\label{eq:opt_bellman}
	\sind{Q}{\star}{h}(s, a) = (\mathbb{T}_h \sind{V}{\star}{h+1})(s, a) , \qquad 
	\sind{V}{\star}{h}(s) = \max_{a\in\mathcal{A}}\sind{Q}{\star}{h}(s, a), \qquad \sind{V}{\star}{H+1}(s) = 0\,.
\end{align}
Accordingly, the optimal policy $\pi^\star$ is the greedy policy with respect to $\{ Q^\star _h \}_{h=1}^H$. 
Hence the Bellman optimality operator $\mathbb{T}_h^\star$ is defined as
\begin{equation*}
	( \mathbb{T}_h^{\star} Q ) (s_h,a_h) = r_h(s_h,a_h) + \mathbb{E}_{s_{h+1} \sim \mathbb{P}_h(\cdot | s_h, a_h)} [\max_{a \in \mathcal{A}} Q(s_{h+1},a)],
	\qquad~ \forall~ Q: \mathcal{S}\times \mathcal{A} \rightarrow [0,H]\,.
\end{equation*}
By definition, the Bellman equation in Eq.~\eqref{eq:opt_bellman} is equivalent to  $Q_h ^\star = \mathbb{T}_h^\star Q_{h+1} ^\star$,  $\forall h \in [H]$.  

In the {\bf online setting}, the goal is to learn the optimal policy $\pi^\star$ by minimizing the cumulative regret under the interaction with the environment over a number of episodes. 
For any policy  $\pi$, the difference between $V_1^\pi$ and $V_1^\star$ quantifies its sub-optimality. 
Thus, after $T$ (fixed but large) episodes, the total (expected) regret is defined as $\text{Regret}(T) = \sum_{t=1}^T \bigl [\sind{V}{\star}{1} (s^t_1) - \sind{V}{\tilde{\pi}^t}{1} (s^t_1)\bigr]$, where $\tilde{\pi}^t$ is the policy executed in the $t$-th episode and   $s_1^t$ is  the initial state.

\subsection{Function spaces}
\label{sec:funcspace}

We give an overview of Besov spaces for deep neural networks and the Barron space for two-layer neural networks. More details refer to Appendix~\ref{app:pribesov}. 
For description simplicity, we focus on the ReLU activation function in this work.

{\bf Besov spaces:} Previous work in approximation theory focuses on the ``smoothness'' of the function, e.g.,  H\"{o}lder spaces \cite{chen2019efficient,fan2020theoretical} and Sobolev spaces \cite{abdeljawad2020approximations}.
Here we consider the concept of $\alpha$-smooth from modulus of smoothness \cite{suzuki2019adaptivity}, \textit{cf.}, \cref{app:pribesov}.

Based on this, we consider a more general function space beyond H\"{o}lder spaces and Sobolev spaces, i.e., Besov spaces \cite{sawano2018theory,suzuki2019adaptivity}, which allows for spatially inhomogeneous smoothness with spikes and jumps.
The Besov space is defined by $\mathcal{B}^\alpha_{p,q}(\mathcal{X}) = \{f \in L^p(\mathcal{X}) \mid \|f\|_{\mathcal{B}_{p,q}^\alpha} < \infty\}$, where the Besov norm is $\|f\|_{\mathcal{B}_{p,q}^\alpha} := \|f\|_{L^p(\mathcal{X})} + |f|_{\mathcal{B}^\alpha_{p,q}}$.
The smoothness parameter $\alpha$ indicates which function at a certain smoothness degree can be represented.
For example, if $\alpha > d/p$, then the related Besov space is continuously embedded in the set of the continuous functions; if $\alpha < d/p$, then the functions in the Besov space are no longer continuous.
The formal definition and relations to H\"{o}lder spaces and Sobolev spaces are deferred to Appendix~\ref{app:pribesov}.

{\bf Barron spaces:} 
A two-layer neural network with $m$ neurons can be represented as $	f(\bm x) = \frac{1}{m} \sum_{k=1}^m b_k\sigma(\bm w_k^{\!\top} \bm x + c_k)$ with the ReLU activation function $\sigma(\cdot)$ used in this work and the neural network parameters $\{ b_k, \bm w_k, c_k \}_{k=1}^m$.
It admits the integral representation $f(\bm x)=\int_{\Omega} b \sigma\left(\bm w^{\!\top} \bm x + c \right) \rho(\mathrm{d} b, \mathrm{d} \bm w, \mathrm{d} c)$, $\bm x \in \mathcal{X}$, where $\Omega = \mathbb{R} \times \mathbb{R}^d \times \mathbb{R}$ and $\rho$ is a probability measure over $\Omega$.
Then the Barron space \cite{weinan2021barron} endowed by the Barron norm is defined as
\begin{equation*}
	\widetilde{\mathcal{P}} = \left\{ f~~\mbox{admits \cref{eq:fbarron}}: \| f \|_{\widetilde{\mathcal{P}}} = \inf_{\rho} \left\{ \mathbb{E}_{\rho} |b| (\| \bm w \|_1 + |c|) \right\} < \infty \right\}\,.
\end{equation*}
The Barron space $\widetilde{\mathcal{P}}$ \cite{weinan2021barron} can be (roughly) equipped with the $\ell_1$-path norm, i.e., $\| f \|_{\widetilde{\mathcal{P}}} \leqslant  \| f \|_{\mathcal{P}}:=\frac{1}{m} \sum_{k=1}^m |b_k| (\| \bm w_k \|_1 + c_k) \leqslant 2 \| f \|_{\widetilde{\mathcal{P}}}$. Accordingly, it is natural to use $\| f \|_{\mathcal{P}}$ to denote the Barron norm, as the discrete version.

The Barron space \cite{weinan2021barron} can be regarded as the \emph{largest} function space for two-layer ReLU neural networks.
Here the ``largest'' terminology \cite{weinan2021barron,weinan2020representation} means that the approximation ability can avoid \emph{curse of dimensionality}, i.e., 1) any function in Barron spaces can be efficiently approximated by two-layer neural networks with bounded norm; 2) any continuous function that can be efficiently approximated by two-layer neural networks with bounded norm belongs to a Barron space. 

We remark that, avoiding curse of dimensionality is important in theory for practical high-dimensional data in RL. However, Besov spaces are too large and thus do not enjoy this property for deep ReLU neural networks.

\section{Algorithm: Value iteration via DNNs under $\epsilon$-greedy exploration}

In this section, we lay out our algorithm~\ref{algo:nnfacounter} via value iteration by DNNs under the $\epsilon$-greedy policy.
Though our value iteration algorithm is different from one gradient-step for deep Q-learning in DQN, it still shares the key spirit with DQN in terms of function approximation via DNNs, $\epsilon$-greedy exploration, and experience replay.

{\bf Function class:} We define the function class $\mathcal{F}$ given by $\mathcal{F} = \mathcal{F}_1 \times \cdots \times \mathcal{F}_H$, including $\mathcal{F}_{\tt SNN}$ for two-layer (Shallow) ReLU neural networks and $\mathcal{F}_{\tt DNN}$ for deep ReLU neural networks as below
\begin{equation}\label{eq2nndef}
	\mathcal{F}_{\tt SNN} =\left\{ f: [0,1]^d \rightarrow [0,H] \Big| f(\bm x) = \frac{1}{m} \sum_{k=1}^m b_k \sigma(\bm w_k^{\!\top} \bm x + c_k), \| f \|_{\mathcal{P}} \leqslant B \right\} \,,
\end{equation}
where $B>0$ is the $\ell_1$-path norm constraint parameter, and deep ReLU neural networks \cite{suzuki2019adaptivity} as
\begin{equation}\label{eqdnndef} 
	\begin{split}
		\mathcal{F}_{\tt DNN}(L,m,S,B)
		& := \bigg\{ f: [0,1]^d \rightarrow [0,H] \Big|	f(\bm x)  = (\bm W^{(L)} \sigma( \cdot) + b^{(L)}) \circ \dots  
		\circ (\bm W^{(1)} \bm x + \bm b^{(1)}) \,, \\
		& \qquad \quad \sum_{i=1}^L (\| \bm W^{(i)} \|_0+ \| \bm b^{(i)} \|_0) \leqslant S,~~
		\max_{i} (\| \bm W^{(i)} \|_\infty \vee \| \bm b^{(i)} \|_\infty) \leqslant B \bigg\}\,,
	\end{split}
\end{equation}
where the weight parameters are $\bm W^{(1)} \in \mathbb{R}^{m \times d}$, $ \bm W^{(l)} \in \mathbb{R}^{m \times m},~\forall l \in \{ 2,3,\dots,L-1\}$, and $\bm W^{(L)} \in \mathbb{R}^m$; the bias parameter are $\bm b^{(l)} \in \mathbb{R}^m,~\forall l \in [L-1]$ and $b^{(L)} \in \mathbb{R}$.
Such sparsely-connected neural networks require most of the network parameters to be zero or non-active, which can be verified \cite{hanin2019deep}.
The depth $L$, the width $m$, the sparsity parameter $S$ and the norm parameter $B$ can be determined later in our proof to achieve good approximation and estimation performance.

{\bf Experience replay:} In our setting, after initialization, at $t$-th episode, at $h$-th time step, we have observed $t -1$ transition tuples, $\{ ( s_h^{\tau}, a_h^{\tau}, s_{h+1}^{\tau} ) \}_{\tau = 1}^{t-1}$ and attempt to estimate $\{ Q^{\star}_h\}_{h=1}^H$ via DNNs.
Note that, at each time step $h$, these $t-1$ transition tuples are neither independent nor identically distributed due to the interaction with value functions and stochastic transition.
To pursue the independence among the transition tuples that is required in our analysis, we follow the \emph{experience replay} scheme \cite{lin1992self} that is successfully applied in DQN \cite{mnih2015human}.
The intuition behind experience replay is to break (or weaken) the temporal dependency among the observations for neural networks training.
When the replay memory is large (e.g., $10^6$ in DQN \cite{mnih2015human}), experience replay is close to sampling independent transitions. 
To be specific, at $t$-th episode, we store transition $\{ (s_h^t, a_h^t, r_h, s^t_{h+1}) \}_{h=1}^H$ in the replay memory $\mathcal{D}$, and then sample a mini-batch of
\emph{independent} observations from $\mathcal{D}$ with $\{ (s_h^{\tau_j}, a_h^{\tau_j}, s_{h+1}^{\tau_j} )  \}_{(j,h) \in [\tilde{t}] \times [H]}$ for DNNs training.
Here the number of mini-batch is denoted as $\tilde{t} := \lceil \varrho t \rceil$ with the mini-batch ratio $\varrho \in (0,1)$, and $\{ \tau_j \}_{j=1}^{\tilde{t}}$ is the index for the mino-batch of $\tilde{t}$ independent samples.
Note that such independence assumption from experience replay is also used in RL theory, e.g., \cite{fan2020theoretical,carvalho2020new} and theoretically demonstrated to be a good de-correlator \cite{di2022analysis}.
In fact, our analysis only requires independence via experience replay, which is still weaker than the standard iid assumption.

\begin{algorithm}[t]
	\caption{Value Iteration via DNNs under $\epsilon$-greedy exploration with experience replay}\label{algo:nnfacounter}
	\begin{algorithmic}[1]
		\STATE{\textbf{Input:} Function class $\mathcal{F}$, the number of episodes $T$, the $\epsilon$-greedy parameter $\epsilon \in (0,1)$, mini-batch ratio $\varrho \in (0,1)$.}
		\STATE{Initialize replay memory $\mathcal{D}$.}
		\FOR{episode $t = 1, \ldots, T$}
		\STATE{Receive the initial state $s^t_1$.}
		\STATE{Set $V_{H+1}^t$ as the zero function.}
		\STATE{Set the minibatch size $\tilde{t} := \lceil \varrho t \rceil$ for experience replay.}
		\FOR{step $h = H, \ldots, 1$}
		\STATE Obtain $\widehat{Q}^t_h := \argmin_{f \in \mathcal{F}} \sum_{j = 1}^{\tilde{t}} \left[ f(s_h^{\tau_j}, a_h^{\tau_j}) - r_h(s_h^{\tau_j}, a_h^{\tau_j}) - V^t_{h+1}(s_{h+1}^{\tau_j}) \right]^2 $.
		\STATE Obtain $Q^t_h := \widehat{Q}^t_h$ and $V_h^t (\cdot) = \max_{a\in \mathcal{A}} Q_h^t(\cdot, a)$.
		\ENDFOR
		\STATE 	//{\tt $\epsilon$-greedy for exploration} \\
		\STATE Take the policy $\{ \tilde{\pi}_h^t\}_{h=1}^H$ to be greedy policy with probability $1- \epsilon$ or any policy with probability $\epsilon$. 
		\FOR{step $h = 1, \ldots, H$}
		\STATE 	Take ${a}^t_h \sim \tilde{\pi}_h^t (\cdot | s_h^t)$ .
		\STATE Observe the reward $r_h(s_h^t, {a}_h^t)$ and obtain the next state $s^t_{h+1}$. 
		\ENDFOR
		\STATE 	//{\tt experience replay} \\
		\STATE{Store transition $\{ (s_h^t, a_h^t, r_h, s^t_{h+1}) \}_{h=1}^H$ in $\mathcal{D}$.}
		\STATE{Sample random mini-batch of transitions from $\mathcal{D}$ with $\tilde{t}$ pairs $\{ (s_h^{\tau_j}, a_h^{\tau_j}, s_{h+1}^{\tau_j} )  \}_{(j,h) \in [\tilde{t}] \times [H]}$.}
		\ENDFOR
	\end{algorithmic}
\end{algorithm}	
\vspace{-0.2cm}

{\bf Value iteration via neural networks:}
In our algorithm, we apply the classical least squares value iteration via neural networks for value function learning \cite{osband2019deep}.
We solve the following least squares regression problem via $\tilde{t}$ independent samples
\begin{equation}\label{eqerm}
	\widehat{Q}^t_h = \argmin_{f \in \mathcal{F}} \widehat{\mathcal{E}}^t_h(f) := \frac{1}{\tilde{t}} \sum_{j = 1}^{\tilde{t}} \left[ f(s_h^{\tau_j}, a_h^{\tau_j}) - r_h(s_h^{\tau_j}, a_h^{\tau_j}) - V^t_{h+1}(s_{h+1}^{\tau_j}) \right]^2\,.
\end{equation}
For ease of simplicity for analyses, we directly assume that the global minima solution of problem~\eqref{eqerm} can be obtained, that follows \cite{schmidt2020nonparametric,chen2019efficient,suzuki2021deep} in deep learning theory.
Nevertheless, our result could be extended to allow small optimization error in each episode that will be discussed in \cref{sec:proofoutline}. 

Besides, we also need the expectation version of $\widehat{\mathcal{E}}^t_h$ in problem~\eqref{eqerm} for our analysis.
Formally, we assume each state-action pair in the mini-batch is sampled from a respective (unknown) probability measure, i.e., $ (s_h^{\tau_j}, a_h^{\tau_j}) \sim \mu_h^{\tau_j}, \forall j \in [\tilde{t}]$, where $\mu^{\tau_j}_h \in \mathscr{P}(\mathcal{S} \times \mathcal{A})$ is from the collection of all probability distribution on $\mathcal{S} \times \mathcal{A}$. 
Taking the averaged measure $\bar{\mu}^{\tilde{t}}_h := \frac{1}{\tilde{t}} \sum_{j=1}^{\tilde{t}} \mu_h^{\tau_j} $, the expectation of $\widehat{\mathcal{E}}^t_h$ is defined as
\begin{equation}\label{eqexpecterm}
	\mathcal{E}^{t}_h(f) = \mathbb{E}_{(s_h, a_h) \sim \bar{\mu}^{\tilde{t}}_h, s_{h+1} \sim \mathbb{P}_h(\cdot| s_h, a_h) } \left[ f(s_h, a_h) - r_h(s_h, a_h) - V^t_{h+1}(s_{h+1}) \right]^2\,.
\end{equation}
Note that, $\widehat{Q}^t_h$ in Eq.~\eqref{eqerm} is not an unbiased estimator of the squared Bellman error minimizer \cite{antos2008learning,duan2021risk}.
Indeed, $\mathcal{E}^{t}_h$ differs from the squared Bellman error because of an extra variance term caused by the stochastic transition \cite{bradtke1996linear}.
This biased estimation issue can be avoided (or alleviated) in practice by introducing target networks in DQN \cite{mnih2013playing}.
Some variants \cite{kim2019deepmellow} of DQN can also reduce the biased estimate and performs well without target networks.
Nevertheless, in our analysis, we center around the uniform bound $\sup_{f \in \mathcal{F}} |\mathcal{E}^{t}_h(f) - \widehat{\mathcal{E}}^{t}_h(f)|$ instead of the Bellman error.

{\bf $\epsilon$-greedy exploration:} In order to work in the online setting, we need to ensure that the learner visits ``good'' state action pairs in the sense that are almost maximizers of the value function for unseen state, \emph{a.k.a.}, exploration. 
In RL theory, a classical way is to design an optimistic estimate of the value function via a bonus function $b_h^t$ \cite{yang2020function,cai2020provably} such that ${Q}^t_h = \min \{\widehat{Q}^t_h + b_h^t, H  \}^+$.
Instead, we employ the $\epsilon$-greedy exploration that follows DQN-like algorithms.
Using the $\epsilon$-greedy exploration will ensure each state-action pair can be visited with positive probability and favor independence among samples.
In our algorithm, we directly set ${Q}^t_h := \min \{\widehat{Q}^t_h, H \}^+$, and then naturally incorporate the truncation operation in neural networks training, see~Eqs.~\eqref{eq2nndef} and~\eqref{eqdnndef}. 

Based on the above description, our algorithm centers around deep neural function approximation via value iteration under the $\epsilon$-greedy exploration and experience replay under the online setting.
This problem setting matches the spirit of practical DQN, which allows for better understanding deep RL.

\section{Main results}
\label{sec:mainres}
This section presents our results for value iteration under deep (as well as two-layer) ReLU neural networks via the Besov spaces and Barron spaces, respectively.
Our theory is based on the independence assumption via experience replay and achieves sublinear regret under the $\epsilon$-greedy exploration.

\subsection{Efficient value iteration via DNNs in Besov spaces}
In this setting, we consider $\widehat{Q}^t_h = \argmin_{f \in \mathcal{F}_{\tt DNN}} \widehat{\mathcal{E}}^t_h(f)$ in Eq.~\eqref{eqerm}, where $\mathcal{F}_{\tt DNN}$ is the function space of deep ReLU neural networks defined in Eq.~\eqref{eqdnndef}.
We make the following assumption on the Besov space $\mathcal{B}$, similar to \cite{jin2020provably,yang2020function}, where the Bellman optimality operator maps any bounded value function to a bounded Besov space ball.
\begin{assumption}\label{ass:besov}
	Let $\widetilde{R} $ be a fixed constant. Define $\mathcal{B}_{\widetilde{R}}  =  \{ f\in \mathcal{B}^{\alpha}_{p,q}(\mathcal{X}) : \| f\|_{\mathcal{B}} \leqslant \widetilde{R} \}$ in the Besov space and assume that for any $h \in [H]$ and $Q \colon \mathcal{S} \times \mathcal{A} \rightarrow [0,H] $, we have $\mathbb{T}_h^\star Q  \in \mathcal{B}_{\widetilde{R}}$. 
\end{assumption}
{\bf Remark:} Due to $Q \in [0,H] $, the radius $\widetilde{R}$ in fact depends on $H$, i.e., $\widetilde{R} \asymp H$. 	

Based on this assumption, we have the following theorem on the regret bound in the Besov space for deep RL under the $\epsilon$-greedy exploration.
\begin{theorem}\label{thm:dnn}
	Under Assumption~\ref{ass:besov} with the smoothness parameter $\alpha > d(1/p - 1/4)_+$ in the Besov space $\mathcal{B}^{\alpha}_{p,q}(\mathcal{X})$, considering value function learning \eqref{eqerm} via DNNs defined by Eq.~\eqref{eqdnndef} in~\cref{algo:nnfacounter} under the $\epsilon$-greedy exploration and the mini-batch ratio $\varrho \in (0,1)$, and taking
	\begin{equation}\label{eq:dewid}
		\mbox{\emph{the depth} $L \asymp \frac{d}{2\alpha +d} \log T$, $\quad$ \emph{the width} $m \asymp \frac{d}{2\alpha +d} T^{\frac{d}{2\alpha +d}} \log T$}\,,
	\end{equation}
	then given a MDP-dependent constant $K \in [1,H]$, for any $\delta \in (0,1)$, the total regret can be upper bounded with probability at least $1 - \delta$
	\begin{equation}\label{eq:regretbes}
		\begin{split}
			\text{Regret}(T)  &\lesssim \left( \frac{\epsilon}{A} \right)^{-\frac{K}{2}} \frac{1}{\sqrt{\varrho}} \left(  H^{\frac{3}{2}} T^{\frac{\alpha + d}{2\alpha + d}} \log^{3} T  + H^2 \sqrt{T}  \sqrt{\log \Big( \frac{2}{\delta} \Big)} \log T \right)  + \epsilon HT + \sqrt{TH^3 \log \Big( \frac{4}{\delta} \Big)} \\
			& \lesssim \widetilde{\mathcal{O}}(H^{\frac{H+4}{H+2}} K^{\frac{2}{K+2}} A^{\frac{K}{K+2}} T^{\frac{\alpha K + (\alpha + d)(K+2)}{(2\alpha + d)(K+2)}}) \,, \quad \mbox{taking $ \epsilon = \mathcal{O}((HK)^{\frac{2}{K+2}} A^{\frac{K}{K+2}} T^{-\frac{2\alpha}{ (2\alpha + d) (K+2) }})$}\,.
		\end{split}
	\end{equation}
\end{theorem}
{\bf Remark:} We make the following remarks.\\
\textit{i}) The constant $K$ describes the ``myopic'' level of MDPs under the $\epsilon$-greedy policy, e.g., the worst case ($K:=H$) under the sparse rewards setting; the benign case $K:=c$ (for some small constant $c$) under the helpful dense rewards setting as discussed in \cite{dann2022guarantees}. 
The exponential dependence on $H$ (in the worst case for any MDP) can be avoided at an additional cost of worsening $T$ dependence.
In fact, whether in the benign/worst case, the sublinear regret is always achieved under some certain $\epsilon$ values in~\cref{eq:regretbes}, which theoretically demonstrates the efficiency of deep RL. Note that the chosen $\epsilon \in (0,1)$ is always satisfied under a large episode $T$. \\
\textit{ii}) Clearly, the regret bound is a non-increasing function of the smoothness parameter $\alpha$, which shows that an easier task (i.e., the target Q function is more smooth) leads to regret bounds with faster rates.
Specially, if we take $\alpha \rightarrow \infty$ (i.e., the target Q function is sufficiently smooth), which holds for linear function approximation
\begin{equation*}
    \text{Regret}(T) \lesssim \widetilde{\mathcal{O}}(H^{\frac{H+4}{H+2}} K^{\frac{2}{K+2}} A^{\frac{K}{K+2}} T^{\frac{K+1}{K+2}} )\,,
\end{equation*}
which recovers the regret bound $\widetilde{\mathcal{O}}(T^{\frac{K+1}{K+2}})$ in \cite[Theorem 3]{dann2022guarantees} via Eluder dimension.
In the best case ($K=1$), our regret bound implies  $\widetilde{\mathcal{O}}(H^{\frac{4}{3}} A^{\frac{1}{3}} T^{\frac{2}{3}} )$ with $H \geqslant 4$, which matches the optimal regret bound for the contextual bandits problem in terms of dependence on $T$ or $A$ under the $\epsilon$-greedy exploration \cite{lattimore2020bandit}. 
In the worst case ($K:=H$), we can still obtain the sublinear regret at a certain $\widetilde{\mathcal{O}}(T^{\frac{H+1}{H+2}})$ rate.

\cref{thm:dnn} demonstrates that the sublinear regret can be achieved by choosing ${\mathcal{O}}(\log T)$ depth and  $\widetilde{\mathcal{O}}(T^{\frac{d}{2\alpha +d}})$ width, but the sublinear regret bound $\widetilde{\mathcal{O}}(T^{\frac{\alpha K + (\alpha + d)(K+2)}{(2\alpha + d)(K+2)}})$ heavily depends on the feature dimension $d$, failing in the \emph{curse of dimensionality}, which appears ineffective on high dimensional data in deep RL.
In the next, we consider the Barron spaces, i.e., the ``largest'' function space for two-layer neural networks to avoid the curse of dimensionality.
In this case, the rate of the sublinear regret can get rid of $d$, which is useful for high dimensional data in practical RL.

\subsection{Efficient value iteration via two-layer neural networks in Barron spaces}
\label{sec:regret2nn}

As mentioned before, Barron spaces are the ``largest'' function space for two-layer neural networks.
In this setting, we consider $\widehat{Q}^t_h = \argmin_{f \in \mathcal{F}_{\tt SNN}} \widehat{\mathcal{E}}^t_h(f)$ in Eq.~\eqref{eqerm}, where $\mathcal{F}_{\tt SNN}$ is the function space of two-layer ReLU neural networks defined in Eq.~\eqref{eq2nndef}.
We give a similar assumption on the Bellman optimality operator in the Barron space.
\begin{assumption}\label{ass:barron}
	Let $\widetilde{R}> 0$ be a fixed constant. Define $\mathcal{P}_{\widetilde{R}}  =  \{ f\in \mathcal{P} : \| f\|_{\mathcal{P}} \leq \widetilde{R} \}$ in the Barron space, and assume that for any $h \in [H]$ and $Q \colon \mathcal{S} \times \mathcal{A} \rightarrow [0,H] $, we have $\mathbb{T}_h^\star Q  \in \mathcal{P}_{\widetilde{R}}$. 
\end{assumption}
Based on this assumption, we have the following regret bounds for two-layer ReLU neural networks.
\begin{theorem}\label{thm:2nn}
	Under Assumption~\ref{ass:barron}, considering value function learning \eqref{eqerm} by two-layer ReLU neural networks with width $m$ and bounded $\ell_1$ norm $B$ defined by Eq.~\eqref{eq2nndef} in Algorithm~\ref{algo:nnfacounter} under the $\epsilon$-greedy exploration and the mini-batch ratio $\varrho \in (0,1)$, then given a MDP-dependent constant $K \in [1,H]$, for any $\delta \in (0,1)$, the total regret can be upper bounded with probability at least $1 - \delta$
	\begin{equation*}
		\begin{split}
			\text{Regret}(T)  & \lesssim \left( \frac{\epsilon}{A} \right)^{-\frac{K}{2}} \!\! \left( \frac{H^2 T^{\frac{3}{4}}}{\sqrt{\varrho}} \left[ B(\log d)^{\frac{1}{4}} \!+\! \log^{\frac{1}{4}} \Big( \frac{4}{\delta} \Big)  \right] \!+\! \frac{H^2T}{\sqrt{m}} \right)  \!+\! \epsilon HT \!+\! \sqrt{TH^3 \log \Big( \frac{4}{\delta} \Big)}\\
			& \lesssim \widetilde{\mathcal{O}}(	H^{\frac{K+4}{K+2}} T^{\frac{2K+3}{2K+4}} ) \,, \quad \mbox{by taking $m = \Omega(\sqrt{T})$ and $\epsilon = \mathcal{O} \left(H^{\frac{2}{K+2}} T^{-\frac{1}{2(K+2)}} \right)$}\,.
		\end{split}
	\end{equation*}
\end{theorem}
{\bf Remark:} In our result, taking $m = \Omega(\sqrt{T})$ is suffice to achieve the sublinear regret bound $\widetilde{\mathcal{O}}(T^{\frac{2K+3}{2K+4}})$, which also gets rid of the feature dimension $d$, allowing for high-dimensional image data in practice.

\section{Discussion on architecture guidelines in deep RL}

In this section, we present a detailed discussion on how our results provide the architecture guidelines in practical deep RL, in the perspective of the width, the depth, and problem-dependent smoothness of the Q function.

{\bf Width-depth and DQN:}
According to \cref{thm:dnn}, the $\mathcal{O}(\log T)$ depth and $\widetilde{O}(T^{\frac{d}{2\alpha + d}})$ width are enough for sublinear regret in deep RL.
Interestingly, we notice that this result is closely matching practical implementation of DQN. For example, the choices of \cite{mnih2015human} $m=512$ and $L=5$ can be explained by our theory, indeed $\log(512) \approx 6$.
Specially, when taking $\alpha \rightarrow \infty$, this setting holds for linear function approximation.  
For two-layer neural networks endowed by the Barron space, the curse of dimensionality in terms of width and regret bound can be avoided in \cref{thm:2nn}, supporting the premise of practical, high-dimensional RL.

{\bf Problem-dependent smoothness and exploration:} The problem-dependent smoothness, determined by $\alpha$, largely affects our regret bounds. 
The difficulty of a task in deep RL can be defined in two views: one is the smoothness of the target Q function; and the other is the degree of exploration.
Intuitively speaking, if a RL task is difficult, then the target Q function is often complicated, and thus admits a relative lower smoothness; or we need conduct more exploration in a complex scenario.
Our results coincide with these two views.
One hand, the regret bound in \cref{thm:dnn} is a non-increasing function of the smoothness parameter $\alpha$. A more difficult task in deep RL (i.e., a smaller $\alpha$) leads to a slower rate of the sublinear regret, which indicates that more episodes are required.
On the other hand, \cref{thm:dnn} shows that the parameter $\epsilon$ is also a non-increasing function of $\alpha$. That means, a more difficult task in deep RL requires a larger $\epsilon$, i.e., we need conduct exploration more frequently. 

Besides, the exploration parameter $\epsilon$ is also affected by $K$ for MDPs with different situations.
For example, compared to the best case $K=1$, more frequent exploration (a larger $\alpha$) is required in MDPs under difficult cases, which coincides with our certain $\epsilon$ value in Theorems~\ref{thm:dnn} and \ref{thm:2nn}. 

{\bf Width and depth trade-off:} Under a limit parameter budget, according to the width-depth ratio $m/L = T^{\frac{d}{2\alpha+d}}$ in \cref{thm:dnn}, our theory indicates that less problem-dependent smoothness of Q-function requires DNNs to be wider.  
In practice, if we work in the limited budget of parameters $N$ in neural networks, e.g., $N \asymp m^2L$, our theory implies that there is a tradeoff between the depth and width on smoothness, i.e., the depth $L:= N^{1/3} T^{-\frac{2d}{3(2\alpha+d)}}$ increasing with $\alpha$ (or $T$) and the width $m = {N}^{1/3} T^{\frac{d}{3(2\alpha+d)}}$ decreasing with $\alpha$ (or $T$).

Besides, according to the width-depth ratio, it can be found that, the change of $\alpha$ leads to less changes on the depth but more changes on the width. 
This shows that width and depth admit different levels of parameter sensitivity under the change of problem-dependent smoothness.

\section{Proof outline}
\label{sec:proofoutline}
In this section, we outline the proof of our theoretical results presented in Section~\ref{sec:mainres}. 
As mentioned before, the technical challenge in our analysis is how to estimate the TD error without bonus function design.
Apart from the regret decomposition, our proof framework includes two main parts: transformation of TD error estimation to generalization bounds, see Figure~\ref{figassum}; and generalization bounds on non-iid data in certain Besov/Barron spaces for TD error analysis, see Figure~\ref{figgen}.
The complete proof is reported in the appendix. 

\begin{figure*}[t]
	\centering
	\AB{Regret\\decomp.\\Lem.~\ref{lem:regret_decomp}}
	{
		statistical error: Lem.~\ref{lemma:bound_mtg} with $\mathcal{O}(\sqrt{H^3T})$\\
		\AB{${\tt Term~(i)}$ $\Leftarrow$ Lem.~\ref{lem:tdgen}}
		{
			generalization\\
		    approximation
		} 
		$\Leftarrow$
		{Lem.~\ref{lem:termitd}: $\| \Gamma^t_h \|^2_{L^2{(\mathrm{d}\bar{\mu}^{\tilde{t}}_h)}}$ $\Leftarrow$  Lem.~\ref{lemmalowerbound}: $\bar{\mu}^{\tilde{t}}_h(\mathcal{C}) > 0$
		}\\
		$\epsilon$-greedy exploration: $\epsilon HT$
	}
	\caption{Proof framework of regret decomposition and transformation of the TD error.}\label{figassum}
	\vspace{-0.3cm}
\end{figure*}

{\bf Regret decomposition:}
This part is standard and commonly studied in RL theory, e.g., \cite{cai2020provably,jin2020provably,yang2020function}.
We briefly include here for self-completeness.
Define the temporal-difference (TD) error as 
\begin{equation}\label{eq:td_error}
	\Gamma^t_h(s,a) =r_h (s,a) + (\mathbb{P}_{h}V^{t}_{h+1})(s,a)  - Q^{t}_h (s,a) = (\mathbb{T}_h  V^{t}_{h+1}) (s,a) -  Q^{t}_h (s,a), \quad \forall (s,a) \in \mathcal{S} \times \mathcal{A}\,,
\end{equation}
where $\Gamma_h^t$ is a function on $\mathcal{S}\times \mathcal{A}$ for all $h \in [H]$ and $t \in [T]$.
Accordingly, the regret can be decomposed into (\emph{c.f.} Lemma~\ref{lem:regret_decomp})
	\begin{equation}\label{eq:regretdec}
			\begin{split}
					\text{Regret}(T) &\leqslant  \underbrace{ \sum_{t=1}^T \sum_{h=1}^H\Big( \mathbb{E}_{\pi^{\star}}[ \Gamma^t_h(s_h,a_h)\,|\, s_1 = s^t_1] 
					-  \Gamma^t_h(s^t_h,a^t_h)\Big)}_{{\tt Term~(i)}} + {\tt Term~(ii)} + \epsilon HT \,,
			\end{split}
		\end{equation}
where the first term relates to the TD error and the second term is the statistical error based on the standard martingale difference sequences, which can be upper bounded by the Hoeffding-Azuma inequality with $\mathcal{O}(\sqrt{H^3T})$ regret (\emph{c.f.} Lemma~\ref{lemma:bound_mtg}).
The last term $\epsilon HT$ is due to the $\epsilon$-greedy exploration.

{\bf Transforming TD error to generalization bounds:} To bound the TD error, we first introduce \cref{lemmalowerbound} with $\bar{\mu}^{\tilde{t}}_h(\mathcal{C}) > 0$, where the event $\mathcal{C}$ denotes that all state-action pairs have been visited at all time steps under the $\epsilon$-greedy policy.
Then we are able to build the connection between ${\tt Term~(i)}$ and $\| \Gamma_{h}^t \|_{L^2(\mathrm{d} \bar{\mu}^{\tilde{t}}_h)}$ in the $L^2{(\mathrm{d}\bar{\mu}^{\tilde{t}}_h)}$-integrable space (\emph{c.f.} Lemma~\ref{lem:termitd}).
After analysis of $\mathcal{E}_h^t(f)$ in \cref{prop:exerm}, we transform the estimation of $\| \Gamma_{h}^t \|_{L^2(\mathrm{d} \bar{\mu}^{\tilde{t}}_h)}$ to the following two terms: generalization error and approximation error, respectively (\emph{c.f.} Lemma~\ref{lem:tdgen})
\begin{equation}\label{eq:tdgen}
		 \| \Gamma^t_h \|^2_{L^2{(\mathrm{d}\bar{\mu}_h^{\tilde{t}})}} \leqslant \Big[ \mathcal{E}^t_h(\widehat{Q}^t_h) - 	\min_{f \in \mathcal{F}} \mathcal{E}^t_h(f) \Big] + \inf_{f \in \mathcal{F}} \| f - \mathbb{T}_h^{\star} Q^t_{h+1} \|^2_{ L^2{(\mathrm{d} \bar{\mu}^{\tilde{t}}_h) }} \,.
\end{equation}
where the first term is the generalization error which we elucidate in the next and the second term is the approximation error and can be considered in an $L^p(\mathcal{X})$ space for Besov spaces in~\cref{coro:tdgen}. 
For example, the approximation error in the Besov space admits the certain $\mathcal{O}(N^{-2\alpha/d})$ rate in \cite{suzuki2019adaptivity} for deep ReLU networks with $L \asymp \log N, ~S \asymp N, ~m \asymp N \log N$.

\begin{figure*}[t]
	\centering
	\AB{}
	{
		Rademacher complexity on non-iid data: Lem.~\ref{lem:expRn} $\Leftarrow$ Lem.~\ref{lem:EfRn} \\
		two-layer NNs: Thm.~\ref{thm:2nn}  $\Leftarrow$ $\mathcal{O}(H^2 B \sqrt{\log d/ n})$ $\Leftarrow$ Lem.~\ref{lem:rade2nnpath}: Rademacher complexity of Barron spaces\\
		DNNs: Thm.~\ref{thm:dnn} $\Leftarrow$ $\mathcal{O}(n^{-\frac{2 \alpha}{2 \alpha  + d}})$ $\Leftarrow$ Prop.~\ref{prop:genbes} $\Leftarrow$ Lem.~\ref{lem:lradebes} on ${\tt LRC}$ for Besov spaces 
	}
	\caption{Proof framework of the TD error via generalization bounds on $n$ non-iid data. We denote ${\tt LRC}$ by local Rademacher complexity for short.}\label{figgen}
	\vspace{-0.2cm}
\end{figure*}

{\bf Generalization bounds on non-iid data:} The key part left is to bound the generalization error on non-i.i.d data for the TD error estimation, see the proof framework in Figure~\ref{figgen}. 
In our proof, we firstly verify that the maximum error in estimating the mean of any function $f \in \mathcal{F}$ can be still bounded by the Rademacher complexity of $\mathcal{F}$ in Lemma~\ref{lem:EfRn}, and then generalization bounds by Rademacher complexity still holds by Lemma~\ref{lem:expRn} via the averaged measure $\bar{\mu}^{\tilde{t}}_h$, which only requires the data to be independent. 
These results can be easily extended to local Rademacher complexity.

For deep neural networks, by computing the local Rademacher complexity of $\mathcal{F}_{\tt DNN}$ in Lemma~\ref{lem:lradebes} and choosing proper neural network parameters in Eq.~\eqref{eqdnndef}, we derive the convergence rate of generalization bounds at a certain $\mathcal{O}(n^{-\frac{2 \alpha}{2 \alpha  + d}})$ rate in Besov spaces (\emph{c.f.} \cref{prop:genbes}) with $n$ non-iid data. 
Combining the result of approximation error and taking the depth and width in \cref{eq:dewid}, ${\tt Term~(i)}$ can be upper bounded with high probability.
Finally we conclude the proof of \cref{thm:dnn} by combining with the statistical error.

For two-layer neural networks, by computing the Rademacher complexity of $\mathcal{F}_{\tt SNN}$ in Lemma~\ref{lem:rade2nnpath}, we obtain the generalization error at a certain $\mathcal{O}(H^2 B \sqrt{\log d/ n})$ convergence rate.
Combining the result of approximation error in Barron spaces with other terms in the regret decomposition, we conclude the proof of Theorem~\ref{thm:2nn}.

{\bf Regret bounds effected by optimization error:} Here we briefly discuss the regret bound affected by a solution (denoted as $\widetilde{Q}_h^t$) that is not a global minimum of problem~\eqref{eqerm}.
Assume that the optimization error is small in the functional view, i.e., $ \| \widetilde{Q}_h^t - \widehat{Q}_h^t \|_{ L^2{(\mathrm{d} \bar{\mu}^{\tilde{t}}_h) }} \leqslant \varepsilon_{\mathrm{opt}}$, that will appear in \cref{eq:tdgen}, and accordingly ${\tt Term~(i)}$ incurs in an extra regret bound $\mathcal{O}(H^2\log T)$ if we take $\varepsilon_{\mathrm{opt}}:=H/\sqrt{\tilde{t}}$.
This condition is fair and reasonable as the optimization error decreases with the mini-batch size $\tilde{t}$ for neural network training but requires a refined analysis under non-iid data \cite{kowshik2021streaming,alacaoglu2022convergence}.

\vspace{-0.2cm}
\section{Conclusion}
\vspace{-0.2cm}

This paper provides an in-depth understanding on neural network function approximation with the $\epsilon$-greedy exploration under the online setting beyond the ``linear'' regime.
Our results provide theoretical guarantees of sublinear regret bounds, and shed light on some guidelines for understanding deep RL in the perspective of the width-depth configuration and the problem-dependent smoothness of RL tasks.

The analysis of this work is built on the $\epsilon$-greedy policy for exploration,  which are satisfied in practical cases when employing DQN. Nevertheless, designing a provably efficient exploration mechanism for deep RL could be an interesting future direction in both practice and theory.
Besides, our theory requires state-action pairs to be independent, which (approximately) holds via experience replay and could be improved by reverse experience replay \cite{agarwal2022online}. 
Furthermore, our work is built on the value iteration based algorithm, which is different from practical DQN that adapts Q-learning via one-step gradient descent. 
Towards a better understanding DQN in terms of Q learning and target networks \cite{zanette2022stabilizing,carvalho2020new} would be an interesting direction.

\section*{Acknowledgement}
The authors would like to thank anonymous reviewers for their constructive suggestions to improve the presentation and point out the independence issue.

This work was supported by SNF project – Deep Optimisation of the Swiss National Science Foundation (SNSF) under grant number 200021\_205011; the Enterprise for Society Center (E4S); the European Research Council (ERC) under the European Union's Horizon 2020 research and innovation programme (grant agreement n°725594 - time-data).

\bibliographystyle{unsrt}
\bibliography{refs}

\begin{thebibliography}{10}

\bibitem{szepesvari2010algorithms}
Csaba Szepesv{\'a}ri.
\newblock Algorithms for reinforcement learning.
\newblock {\em Synthesis lectures on artificial intelligence and machine
  learning}, 4(1):1--103, 2010.

\bibitem{sutton2018reinforcement}
Richard~S Sutton and Andrew~G Barto.
\newblock {\em Reinforcement learning: An introduction}.
\newblock MIT press, 2018.

\bibitem{fujimoto2018addressing}
Scott Fujimoto, Herke Hoof, and David Meger.
\newblock Addressing function approximation error in actor-critic methods.
\newblock In {\em International conference on machine learning}, pages
  1587--1596. PMLR, 2018.

\bibitem{mnih2015human}
Volodymyr Mnih, Koray Kavukcuoglu, David Silver, Andrei~A Rusu, Joel Veness,
  Marc~G Bellemare, Alex Graves, Martin Riedmiller, Andreas~K Fidjeland, Georg
  Ostrovski, et~al.
\newblock Human-level control through deep reinforcement learning.
\newblock {\em nature}, 518(7540):529--533, 2015.

\bibitem{silver2016mastering}
David Silver, Aja Huang, Chris~J Maddison, Arthur Guez, Laurent Sifre, George
  Van Den~Driessche, Julian Schrittwieser, Ioannis Antonoglou, Veda
  Panneershelvam, Marc Lanctot, et~al.
\newblock Mastering the game of go with deep neural networks and tree search.
\newblock {\em nature}, 529(7587):484--489, 2016.

\bibitem{lecun2015}
LeCun Yann, Bengio Yoshua, and Hinton Geoffrey.
\newblock Deep learning.
\newblock {\em Nature}, 521(7553):436--444, 2015.

\bibitem{jin2020provably}
Chi Jin, Zhuoran Yang, Zhaoran Wang, and Michael~I Jordan.
\newblock Provably efficient reinforcement learning with linear function
  approximation.
\newblock In {\em Conference on Learning Theory}, pages 2137--2143. PMLR, 2020.

\bibitem{wang2020optimism}
Yining Wang, Ruosong Wang, Simon~Shaolei Du, and Akshay Krishnamurthy.
\newblock Optimism in reinforcement learning with generalized linear function
  approximation.
\newblock In {\em International Conference on Learning Representations}, 2020.

\bibitem{ayoub2020model}
Alex Ayoub, Zeyu Jia, Csaba Szepesvari, Mengdi Wang, and Lin Yang.
\newblock Model-based reinforcement learning with value-targeted regression.
\newblock In {\em International Conference on Machine Learning}, pages
  463--474. PMLR, 2020.

\bibitem{zhou2021nearly}
Dongruo Zhou, Quanquan Gu, and Csaba Szepesvari.
\newblock Nearly minimax optimal reinforcement learning for linear mixture
  markov decision processes.
\newblock In {\em Conference on Learning Theory}, pages 4532--4576. PMLR, 2021.

\bibitem{jacot2018neural}
Arthur Jacot, Franck Gabriel, and Cl{\'e}ment Hongler.
\newblock Neural tangent kernel: Convergence and generalization in neural
  networks.
\newblock In {\em Advances in Neural Information Processing Systems}, pages
  8571--8580, 2018.

\bibitem{yang2020function}
Zhuoran Yang, Chi Jin, Zhaoran Wang, Mengdi Wang, and Michael~I Jordan.
\newblock On function approximation in reinforcement learning: Optimism in the
  face of large state spaces.
\newblock In {\em Advances in Neural Information Processing Systems}, 2020.

\bibitem{jiang2017contextual}
Nan Jiang, Akshay Krishnamurthy, Alekh Agarwal, John Langford, and Robert~E
  Schapire.
\newblock Contextual decision processes with low bellman rank are
  pac-learnable.
\newblock In {\em International Conference on Machine Learning}, pages
  1704--1713. PMLR, 2017.

\bibitem{dong2020root}
Kefan Dong, Jian Peng, Yining Wang, and Yuan Zhou.
\newblock Root-n-regret for learning in markov decision processes with function
  approximation and low bellman rank.
\newblock In {\em Conference on Learning Theory}, pages 1554--1557. PMLR, 2020.

\bibitem{russo2013eluder}
Daniel Russo and Benjamin Van~Roy.
\newblock Eluder dimension and the sample complexity of optimistic exploration.
\newblock {\em Advances in Neural Information Processing Systems}, 26, 2013.

\bibitem{wang2020reinforcement}
Ruosong Wang, Russ~R Salakhutdinov, and Lin Yang.
\newblock Reinforcement learning with general value function approximation:
  Provably efficient approach via bounded eluder dimension.
\newblock In {\em Advances in Neural Information Processing Systems},
  volume~33, pages 6123--6135, 2020.

\bibitem{ishfaq2021randomized}
Haque Ishfaq, Qiwen Cui, Viet Nguyen, Alex Ayoub, Zhuoran Yang, Zhaoran Wang,
  Doina Precup, and Lin Yang.
\newblock Randomized exploration in reinforcement learning with general value
  function approximation.
\newblock In {\em International Conference on Machine Learning}, pages
  4607--4616. PMLR, 2021.

\bibitem{chizat2019lazy}
Lenaic Chizat, Edouard Oyallon, and Francis Bach.
\newblock On lazy training in differentiable programming.
\newblock In {\em Advances in Neural Information Processing Systems}, pages
  2933--2943, 2019.

\bibitem{lee2019wide}
Jaehoon Lee, Lechao Xiao, Samuel Schoenholz, Yasaman Bahri, Roman Novak, Jascha
  Sohl-Dickstein, and Jeffrey Pennington.
\newblock Wide neural networks of any depth evolve as linear models under
  gradient descent.
\newblock In {\em Advances in Neural Information Processing Systems}, pages
  8570--8581, 2019.

\bibitem{woodworth2020kernel}
Blake Woodworth, Suriya Gunasekar, Jason~D Lee, Edward Moroshko, Pedro
  Savarese, Itay Golan, Daniel Soudry, and Nathan Srebro.
\newblock Kernel and rich regimes in overparametrized models.
\newblock In {\em Conference on Learning Theory}, pages 3635--3673. PMLR, 2020.

\bibitem{geiger2020disentangling}
Mario Geiger, Stefano Spigler, Arthur Jacot, and Matthieu Wyart.
\newblock Disentangling feature and lazy training in deep neural networks.
\newblock {\em Journal of Statistical Mechanics: Theory and Experiment},
  2020(11):113301, 2020.

\bibitem{bach2017breaking}
Francis Bach.
\newblock Breaking the curse of dimensionality with convex neural networks.
\newblock {\em Journal of Machine Learning Research}, 18(1):629--681, 2017.

\bibitem{yehudai2019power}
Gilad Yehudai and Ohad Shamir.
\newblock On the power and limitations of random features for understanding
  neural networks.
\newblock In {\em Advances in Neural Information Processing Systems}, pages
  6594--6604, 2019.

\bibitem{celentano2021minimum}
Michael Celentano, Theodor Misiakiewicz, and Andrea Montanari.
\newblock Minimum complexity interpolation in random features models.
\newblock {\em arXiv preprint arXiv:2103.15996}, 2021.

\bibitem{huang2021going}
Baihe Huang, Kaixuan Huang, Sham Kakade, Jason~D Lee, Qi~Lei, Runzhe Wang, and
  Jiaqi Yang.
\newblock Going beyond linear rl: Sample efficient neural function
  approximation.
\newblock In {\em Advances in Neural Information Processing Systems}, 2021.

\bibitem{li2021eluder}
Gene Li, Pritish Kamath, Dylan~J Foster, and Nathan Srebro.
\newblock Eluder dimension and generalized rank.
\newblock {\em arXiv preprint arXiv:2104.06970}, 2021.

\bibitem{dong2021provable}
Kefan Dong, Jiaqi Yang, and Tengyu Ma.
\newblock Provable model-based nonlinear bandit and reinforcement learning:
  Shelve optimism, embrace virtual curvature.
\newblock In {\em Advances in Neural Information Processing Systems},
  volume~34, 2021.

\bibitem{osband2019deep}
Ian Osband, Benjamin Van~Roy, Daniel~J Russo, Zheng Wen, et~al.
\newblock Deep exploration via randomized value functions.
\newblock {\em Journal of Machine Learning Research}, 20(124):1--62, 2019.

\bibitem{hester2018deep}
Todd Hester, Matej Vecerik, Olivier Pietquin, Marc Lanctot, Tom Schaul, Bilal
  Piot, Dan Horgan, John Quan, Andrew Sendonaris, Ian Osband, et~al.
\newblock Deep q-learning from demonstrations.
\newblock In {\em Proceedings of the AAAI Conference on Artificial
  Intelligence}, volume~32, 2018.

\bibitem{suzuki2019adaptivity}
Taiji Suzuki.
\newblock Adaptivity of deep {ReLU} network for learning in {Besov and mixed
  smooth Besov} spaces: optimal rate and curse of dimensionality.
\newblock In {\em International Conference on Learning Representations}, 2019.

\bibitem{weinan2021barron}
Weinan E, Chao Ma, and Lei Wu.
\newblock The barron space and the flow-induced function spaces for neural
  network models.
\newblock {\em Constructive Approximation}, pages 1--38, 2021.

\bibitem{yang2020reinforcement}
Lin Yang and Mengdi Wang.
\newblock Reinforcement learning in feature space: Matrix bandit, kernels, and
  regret bound.
\newblock In {\em International Conference on Machine Learning}, pages
  10746--10756. PMLR, 2020.

\bibitem{russo2018tutorial}
Daniel~J Russo, Benjamin Van~Roy, Abbas Kazerouni, Ian Osband, Zheng Wen,
  et~al.
\newblock A tutorial on thompson sampling.
\newblock {\em Foundations and Trends{\textregistered} in Machine Learning},
  11(1):1--96, 2018.

\bibitem{agrawal2012analysis}
Shipra Agrawal and Navin Goyal.
\newblock Analysis of thompson sampling for the multi-armed bandit problem.
\newblock In {\em Conference on Learning Theory}, pages 1--39. JMLR Workshop
  and Conference Proceedings, 2012.

\bibitem{zanette2020frequentist}
Andrea Zanette, David Brandfonbrener, Emma Brunskill, Matteo Pirotta, and
  Alessandro Lazaric.
\newblock Frequentist regret bounds for randomized least-squares value
  iteration.
\newblock In {\em International Conference on Artificial Intelligence and
  Statistics}, pages 1954--1964. PMLR, 2020.

\bibitem{lin1992self}
Long-Ji Lin.
\newblock Self-improving reactive agents based on reinforcement learning,
  planning and teaching.
\newblock {\em Machine learning}, 8(3):293--321, 1992.

\bibitem{fan2020theoretical}
Jianqing Fan, Zhaoran Wang, Yuchen Xie, and Zhuoran Yang.
\newblock A theoretical analysis of deep q-learning.
\newblock In {\em Learning for Dynamics and Control}, pages 486--489, 2020.

\bibitem{long2021L2}
Jihao Long, Jiequn Han, et~al.
\newblock An $ l^2$ analysis of reinforcement learning in high dimensions with
  kernel and neural network approximation.
\newblock {\em arXiv preprint arXiv:2104.07794}, 2021.

\bibitem{nguyen2021sample}
Thanh Nguyen-Tang, Sunil Gupta, Hung Tran-The, and Svetha Venkatesh.
\newblock Sample complexity of offline reinforcement learning with deep relu
  networks.
\newblock {\em arXiv preprint arXiv:2103.06671}, 2021.

\bibitem{riedmiller2005neural}
Martin Riedmiller.
\newblock Neural fitted {Q} iteration--first experiences with a data efficient
  neural reinforcement learning method.
\newblock In {\em European Conference on Machine Learning}, pages 317--328.
  Springer, 2005.

\bibitem{xu2020finite}
Pan Xu and Quanquan Gu.
\newblock A finite-time analysis of q-learning with neural network function
  approximation.
\newblock In {\em International Conference on Machine Learning}, pages
  10555--10565. PMLR, 2020.

\bibitem{jin2018q}
Chi Jin, Zeyuan Allen-Zhu, Sebastien Bubeck, and Michael~I Jordan.
\newblock Is q-learning provably efficient?
\newblock In {\em Advances in neural information processing systems}, 2018.

\bibitem{zanette2022stabilizing}
Andrea Zanette and Martin~J Wainwright.
\newblock Stabilizing {Q}-learning with linear architectures for provably
  efficient learning.
\newblock {\em arXiv preprint arXiv:2206.00796}, 2022.

\bibitem{carvalho2020new}
Diogo Carvalho, Francisco~S Melo, and Pedro Santos.
\newblock A new convergent variant of {Q}-learning with linear function
  approximation.
\newblock In {\em Advances in Neural Information Processing Systems},
  volume~33, pages 19412--19421, 2020.

\bibitem{agarwal2021online}
Naman Agarwal, Syomantak Chaudhuri, Prateek Jain, Dheeraj Nagaraj, and Praneeth
  Netrapalli.
\newblock Online target {Q}-learning with reverse experience replay:
  Efficiently finding the optimal policy for linear mdps.
\newblock {\em arXiv preprint arXiv:2110.08440}, 2021.

\bibitem{szlak2021convergence}
Liran Szlak and Ohad Shamir.
\newblock Convergence results for {Q}-learning with experience replay.
\newblock {\em arXiv preprint arXiv:2112.04213}, 2021.

\bibitem{zanette2020learning}
Andrea Zanette, Alessandro Lazaric, Mykel Kochenderfer, and Emma Brunskill.
\newblock Learning near optimal policies with low inherent bellman error.
\newblock In {\em International Conference on Machine Learning}, pages
  10978--10989. PMLR, 2020.

\bibitem{hu2022nearly}
Pihe Hu, Yu~Chen, and Longbo Huang.
\newblock Nearly minimax optimal reinforcement learning with linear function
  approximation.
\newblock In {\em International Conference on Machine Learning}, pages
  8971--9019, 2022.

\bibitem{dann2022guarantees}
Chris Dann, Yishay Mansour, Mehryar Mohri, Ayush Sekhari, and Karthik
  Sridharan.
\newblock Guarantees for epsilon-greedy reinforcement learning with function
  approximation.
\newblock In {\em International Conference on Machine Learning}, pages
  4666--4689, 2022.

\bibitem{jin2021bellman}
Chi Jin, Qinghua Liu, and Sobhan Miryoosefi.
\newblock Bellman eluder dimension: New rich classes of rl problems, and
  sample-efficient algorithms.
\newblock {\em Advances in Neural Information Processing Systems}, 34, 2021.

\bibitem{puterman2014markov}
Martin~L Puterman.
\newblock {\em Markov decision processes: discrete stochastic dynamic
  programming}.
\newblock John Wiley \& Sons, 2014.

\bibitem{chen2019efficient}
Minshuo Chen, Haoming Jiang, Wenjing Liao, and Tuo Zhao.
\newblock Efficient approximation of deep {ReLU} networks for functions on low
  dimensional manifolds.
\newblock {\em Advances in neural information processing systems},
  32:8174--8184, 2019.

\bibitem{abdeljawad2020approximations}
Ahmed Abdeljawad and Philipp Grohs.
\newblock Approximations with deep neural networks in sobolev time-space.
\newblock {\em arXiv preprint arXiv:2101.06115}, 2020.

\bibitem{sawano2018theory}
Yoshihiro Sawano.
\newblock {\em Theory of Besov spaces}, volume~56.
\newblock Springer, 2018.

\bibitem{weinan2020representation}
Weinan E and Stephan Wojtowytsch.
\newblock Representation formulas and pointwise properties for barron
  functions.
\newblock {\em arXiv preprint arXiv:2006.05982}, 2020.

\bibitem{hanin2019deep}
Boris Hanin and David Rolnick.
\newblock Deep relu networks have surprisingly few activation patterns.
\newblock {\em Advances in Neural Information Processing Systems}, 32, 2019.

\bibitem{di2022analysis}
Shirli Di-Castro, Shie Mannor, and Dotan Di~Castro.
\newblock Analysis of stochastic processes through replay buffers.
\newblock In {\em International Conference on Machine Learning}, pages
  5039--5060. PMLR, 2022.

\bibitem{schmidt2020nonparametric}
Johannes Schmidt-Hieber.
\newblock Nonparametric regression using deep neural networks with {ReLU}
  activation function.
\newblock {\em Annals of Statistics}, 48(4):1875--1897, 2020.

\bibitem{suzuki2021deep}
Taiji Suzuki and Atsushi Nitanda.
\newblock Deep learning is adaptive to intrinsic dimensionality of model
  smoothness in anisotropic besov space.
\newblock In {\em Advances in Neural Information Processing Systems}, 2021.

\bibitem{antos2008learning}
Andr{\'a}s Antos, Csaba Szepesv{\'a}ri, and R{\'e}mi Munos.
\newblock Learning near-optimal policies with bellman-residual minimization
  based fitted policy iteration and a single sample path.
\newblock {\em Machine Learning}, 71(1):89--129, 2008.

\bibitem{duan2021risk}
Yaqi Duan, Chi Jin, and Zhiyuan Li.
\newblock Risk bounds and rademacher complexity in batch reinforcement
  learning.
\newblock In {\em International Conference on Machine Learning}, pages
  2892--2902. PMLR, 2021.

\bibitem{bradtke1996linear}
Steven~J Bradtke and Andrew~G Barto.
\newblock Linear least-squares algorithms for temporal difference learning.
\newblock {\em Machine learning}, 22(1):33--57, 1996.

\bibitem{mnih2013playing}
Volodymyr Mnih, Koray Kavukcuoglu, David Silver, Alex Graves, Ioannis
  Antonoglou, Daan Wierstra, and Martin Riedmiller.
\newblock Playing atari with deep reinforcement learning.
\newblock {\em arXiv preprint arXiv:1312.5602}, 2013.

\bibitem{kim2019deepmellow}
Seungchan Kim, Kavosh Asadi, Michael Littman, and George Konidaris.
\newblock Deepmellow: removing the need for a target network in deep
  q-learning.
\newblock In {\em International Joint Conference on Artificial Intelligence},
  2019.

\bibitem{cai2020provably}
Qi~Cai, Zhuoran Yang, Chi Jin, and Zhaoran Wang.
\newblock Provably efficient exploration in policy optimization.
\newblock In {\em International Conference on Machine Learning}, pages
  1283--1294, 2020.

\bibitem{lattimore2020bandit}
Tor Lattimore and Csaba Szepesv{\'a}ri.
\newblock {\em Bandit algorithms}.
\newblock Cambridge University Press, 2020.

\bibitem{kowshik2021streaming}
Suhas Kowshik, Dheeraj Nagaraj, Prateek Jain, and Praneeth Netrapalli.
\newblock Streaming linear system identification with reverse experience
  replay.
\newblock {\em Advances in Neural Information Processing Systems},
  34:30140--30152, 2021.

\bibitem{alacaoglu2022convergence}
Ahmet Alacaoglu and Hanbaek Lyu.
\newblock Convergence and complexity of stochastic subgradient methods with
  dependent data for nonconvex optimization.
\newblock {\em arXiv preprint arXiv:2203.15797}, 2022.

\bibitem{agarwal2022online}
Naman Agarwal, Syomantak Chaudhuri, Prateek Jain, Dheeraj Nagaraj, and Praneeth
  Netrapalli.
\newblock Online target {Q-learning} with reverse experience replay:
  Efficiently finding the optimal policy for linear {MDPs}.
\newblock In {\em International Conference on Learning Representations}, 2022.

\bibitem{neyshabur2015norm}
Behnam Neyshabur, Ryota Tomioka, and Nathan Srebro.
\newblock Norm-based capacity control in neural networks.
\newblock In {\em Conference on Learning Theory}, pages 1376--1401. PMLR, 2015.

\bibitem{mohri2018foundations}
Mehryar Mohri, Afshin Rostamizadeh, and Ameet Talwalkar.
\newblock {\em Foundations of machine learning}.
\newblock MIT press, 2018.

\bibitem{wainwright2019high}
Martin~J Wainwright.
\newblock {\em High-dimensional statistics: A non-asymptotic viewpoint},
  volume~48.
\newblock Cambridge University Press, 2019.

\bibitem{mendelson2002improving}
Shahar Mendelson.
\newblock Improving the sample complexity using global data.
\newblock {\em IEEE transactions on Information Theory}, 48(7):1977--1991,
  2002.

\bibitem{shalev2014understanding}
Shai Shalev-Shwartz and Shai Ben-David.
\newblock {\em Understanding machine learning: From theory to algorithms}.
\newblock Cambridge University Press, 2014.

\bibitem{bartlett2005local}
Peter~L Bartlett, Olivier Bousquet, and Shahar Mendelson.
\newblock Local rademacher complexities.
\newblock {\em Annals of Statistics}, 33(4):1497--1537, 2005.

\bibitem{lei2016local}
Yunwen Lei, Lixin Ding, and Yingzhou Bi.
\newblock Local rademacher complexity bounds based on covering numbers.
\newblock {\em Neurocomputing}, 218:320--330, 2016.

\end{thebibliography}
\newpage	
%%%%%%%%%%%%%%%%%%%%%%%%%%%%%%%%%%%%%%%%%%%%%%%%%%%%%%%%%%%%
%%%%%%%%%%%%%%%%%%%%%%%%%%%%%%%%%%%%%%%%%%%%%%%%%%%%%%%%%%%%
\section*{Checklist}

\begin{enumerate}

	\item For all authors...
	\begin{enumerate}
		\item Do the main claims made in the abstract and introduction accurately reflect the paper's contributions and scope?
		\answerYes
		\item Did you describe the limitations of your work?
		\answerYes{We clearly discuss the limitation of this work in the Conclusion section.}
		\item Did you discuss any potential negative societal impacts of your work?
		\answerNo{Our work is theoretical and generally will have no negative societal impacts.}
		\item Have you read the ethics review guidelines and ensured that your paper conforms to them?
		\answerYes{}
	\end{enumerate}

	\item If you are including theoretical results...
	\begin{enumerate}
		\item Did you state the full set of assumptions of all theoretical results?
		\answerYes{The assumptions are clearly stated and well discussed.}
		\item Did you include complete proofs of all theoretical results?
		\answerYes{All of the proofs can be found in the Appendix.}
	\end{enumerate}

	\item If you ran experiments...
	\begin{enumerate}
		\item Did you include the code, data, and instructions needed to reproduce the main experimental results (either in the supplemental material or as a URL)?
		\answerNA{}
		\item Did you specify all the training details (e.g., data splits, hyperparameters, how they were chosen)?
		\answerNA{}
		\item Did you report error bars (e.g., with respect to the random seed after running experiments multiple times)?
		\answerNA{}
		\item Did you include the total amount of compute and the type of resources used (e.g., type of GPUs, internal cluster, or cloud provider)?
		\answerNA{}
	\end{enumerate}

	\item If you are using existing assets (e.g., code, data, models) or curating/releasing new assets...
	\begin{enumerate}
		\item If your work uses existing assets, did you cite the creators?
		\answerNA{}
		\item Did you mention the license of the assets?
		\answerNA{}
		\item Did you include any new assets either in the supplemental material or as a URL?
		\answerNA{}
		\item Did you discuss whether and how consent was obtained from people whose data you're using/curating?
		\answerNA{}
		\item Did you discuss whether the data you are using/curating contains personally identifiable information or offensive content?
		\answerNA{}
	\end{enumerate}

	\item If you used crowdsourcing or conducted research with human subjects...
	\begin{enumerate}
		\item Did you include the full text of instructions given to participants and screenshots, if applicable?
		\answerNA{}
		\item Did you describe any potential participant risks, with links to Institutional Review Board (IRB) approvals, if applicable?
		\answerNA{}
		\item Did you include the estimated hourly wage paid to participants and the total amount spent on participant compensation?
		\answerNA{}
	\end{enumerate}

\end{enumerate}
%%%%%%%%%%%%%%%%%%%%%%%%%%%%%%%%%%%%%%%%%%%%%%%%%%%%%%%%%%%%

\newpage

\appendix

The appendix is organized as follows
\begin{itemize}
    \item Appendix~\ref{app:pribesov}: preliminaries on Besov spaces and Barron spaces;
    \item Appendix~\ref{app:regret_decomp}: proofs related to regret decomposition;
    \item Appendix~\ref{app:termitd}: proofs related to the temporal difference error and generalization error;
    \item Appendix~\ref{app:genealization}: proofs related to generalization bounds on non-iid data;
    \item Appendix~\ref{app:dnn}: proofs related to sublinear regret bounds for deep ReLU neural networks endowed by Besov spaces;
    \item Appendix~\ref{app:2nn}: proofs related to sublinear regret bounds for two-layer neural networks endowed by Barron spaces.
\end{itemize}

	\section{Preliminaries: Besov spaces and Barron spaces}
\label{app:pribesov}

In this section, we give an overview of Besov spaces for deep ReLU neural networks and the Barron spaces for two-layer ReLU neural networks.

\subsection{Besov spaces} 
Here we briefly introduce a general function space for deep ReLU neural networks according to the  ``smoothness'' of the function, i.e., Besov spaces.

To define Besov functions, we need introduce the modulus of smoothness.
\begin{definition}\cite[modulus of smoothness]{suzuki2019adaptivity}
	For a function $f \in L^p(\mathcal{X})$ with some $p \in (0,\infty]$, the $k$-th modulus of smoothness of $f$ is defined by
	\begin{equation*}
	    w_{k,p}(f,t) = \sup_{\bm h \in \mathbb{R}^d:~\| \bm h\|_2  \leqslant t} \|\Delta_{\bm h}^k (f)\|_{p}\,,
	\end{equation*}
	with
	\begin{equation*}
		\Delta_{\bm h}^k (f)(\bm x) = 
		\begin{cases} \sum_{j=0}^k {k \choose j} (-1)^{k-j} f(\bm x+j \bm h)~& \text{if}~ \bm x \in \mathcal{X}, ~ \bm x+k \bm h\in \mathcal{X}\,,\\
			0 & \text{otherwise}.
		\end{cases}
	\end{equation*}
\end{definition}
The quantity $\Delta_{\bm h}^k(f)$ captures the local oscillation of function $f$ that is not necessarily differentiable.
Based on this, the Besov space is defined as below.

\begin{definition}\cite[Besov space $\mathcal{B}^\alpha_{p,q}(\mathcal{X})$]{sawano2018theory,suzuki2019adaptivity}
		For $0 < p,q \leqslant \infty$, the smoothness parameter $\alpha > 0$, $k:= \lfloor \alpha \rfloor + 1$, define the semi-norm $|\cdot|_{\mathcal{B}^\alpha_{p,q}}$ as 
	\begin{equation*}
		|f|_{\mathcal{B}^\alpha_{p,q}} :=  
		\begin{cases}
			\left(\int_0^\infty (t^{-\alpha} w_{k,p}(f,t))^q \frac{\mathrm{d} t}{t} \right)^{\frac{1}{q}} & (q < \infty)\,, \\
			\sup_{t > 0} t^{-\alpha} w_{k,p}(f,t)  & (q = \infty)\,.
		\end{cases}
	\end{equation*}
	The norm of the Besov space $\mathcal{B}_{p,q}^\alpha(\mathcal{X})$ is defined by $\|f\|_{\mathcal{B}_{p,q}^\alpha} := \|f\|_{L^p(\mathcal{X})} + |f|_{\mathcal{B}^\alpha_{p,q}}$,	and the Besov space is $\mathcal{B}^\alpha_{p,q}(\mathcal{X}) = \{f \in L^p(\mathcal{X}) \mid \|f\|_{\mathcal{B}_{p,q}^\alpha} < \infty\}$.
\end{definition}

The smoothness parameter $\alpha$ indicates which function at a certain smoothness degree can be represented.
For example, if $\alpha > d/p$, then the related Besov space is continuously embedded in the set of the continuous functions.
However, if $\alpha < d/p$, then the functions in the Besov space are no longer continuous.
In particular, the Besov space reduces to the H\"{o}lder space ${\tt C}^{\alpha}$ when $p = q = \infty$ and $\alpha$ is a positive non-integer; degenerates to the Sobolev space ${\tt W}^{\alpha}_2$ when $p=q=2$ and $\alpha$ is a positive integer.
The Besov space is more general than these two spaces as it allows for spatially inhomogeneous smoothness with spikes and jumps.
More properties of Besov spaces and relations to other function spaces refer to \cite{suzuki2019adaptivity} for details.

\subsection{Barron spaces}
The study for deep ReLU neural networks is endowed by Besov spaces, but the complete of function space for deep ReLU neural networks to avoid the \emph{curse of dimensionality} is still open. 
Luckily, the complete of function space for two-layer neural networks can be conducted by Barron spaces. 
Here we briefly introduce the basic definition and property of Barron spaces \cite{weinan2020representation,weinan2021barron}.

We consider a typical two-layer neural network $	f(\bm x) = \frac{1}{m} \sum_{k=1}^m b_k\sigma(\bm w_k^{\!\top} \bm x + c_k) $, 
where $m$ is the number of neurons in the hidden layer and $\sigma(x) = \max\{x,0\}$ is the ReLU activation function used in this work.
Accordingly, the two-layer neural network admits the following representation
\begin{equation}\label{eq:fbarron}
f(\bm x)=\int_{\Omega} b \sigma\left(\bm w^{\!\top} \bm x + c \right) \rho(\mathrm{d} b, \mathrm{d} \bm w, \mathrm{d} c), \quad \bm x \in \mathcal{X}\,,
\end{equation}
where $\Omega = \mathbb{R} \times \mathbb{R}^d \times \mathbb{R}$ and $\rho$ is a probability measure over $\Omega$.
Then the Barron space \cite{weinan2021barron} endowed by the $p$-Barron norm with $p \in [1, +\infty]$ is defined as
\begin{equation*}
	\widetilde{\mathcal{P}}_p = \left\{ f~~\mbox{admits \cref{eq:fbarron}}: \| f \|_{\widetilde{\mathcal{P}}_p} = \inf_{\rho} \left\{ \mathbb{E}_{\rho} |b|^p (\| \bm w \|_1 + |c|)^p \right\}^{1/p} < \infty \right\}\,.
\end{equation*}
Specifically, when using ReLU, these function spaces under different $p$ are the same, i.e., $\widetilde{\mathcal{P}}_1 = \widetilde{\mathcal{P}}_2 = \cdots = \widetilde{\mathcal{P}}_{\infty}$, and thus we directly use $\widetilde{\mathcal{P}}$ for short.
The is the main reason why we study ReLU activation functions in this work.
Besides, the Barron norm is close to the $\ell_1$-path norm \cite{neyshabur2015norm}
\begin{equation*}
	\| f \|_{\widetilde{\mathcal{P}}} \leqslant  \| f \|_{\mathcal{P}}:=\frac{1}{m} \sum_{k=1}^m |b_k| (\| \bm w_k \|_1 + |c_k|) \leqslant 2 \| f \|_{\widetilde{\mathcal{P}}}\,.
\end{equation*}
Based on this, for description simplicity, we do not strictly distinguish the Barron norm and the $\ell_1$-path norm, and regard $\| f \|_{\mathcal{P}}$ as the discrete version of the Barron norm.

As suggested by \cite{weinan2020representation}, Barron space can be regarded as the \emph{largest} function space for two layer neural networks in two folds \cite{weinan2021barron}:
1)  \emph{direct approximation:} Any function in Barron spaces can be efficiently approximated by two-layer neural networks with bounded $\ell_1$ path norm at $\mathcal{O}(1/m)$ rate without \emph{curse of dimensionality};
2) \emph{inverse approximation:} Any continuous function that can be efficiently approximated by two-layer neural networks with bounded $\ell_1$-path norm belongs to a Barron space.

\section{Regret decomposition}
	\label{app:regret_decomp}
	We present the regret decomposition under the $\epsilon$-greedy policy by constructing the martingale difference sequence and giving error bounds for this.
	Apart from an extra $\epsilon HT$ regret, this decomposition result appears in \cite{cai2020provably,yang2020function}, and we include them here just for self-completeness.
	
	To establish the regret decomposition, we need some notations.
	 Remember the definition of the regret, $\tilde{\pi}^t$ is the $\epsilon$-greedy policy and $\pi^t$ is the greedy policy at the $t$-th episode, and then we have
        \begin{equation*}
        \begin{split}
           \text{Regret}(T) & = \sum_{t=1}^T \bigl [\sind{V}{\star}{1} (s^t_1) - \sind{V}{{\pi}^t}{1} (s^t_1)\bigr] + \sum_{t=1}^T \bigl [\sind{V}{{\pi}^t}{1} (s^t_1) - \sind{V}{\tilde{\pi}^t}{1} (s^t_1)\bigr]  \\
           & \leqslant \sum_{t=1}^T \bigl [\sind{V}{\star}{1} (s^t_1) - \sind{V}{{\pi}^t}{1} (s^t_1)\bigr] + \epsilon H T \,,
        \end{split}
        \end{equation*}
        where $\epsilon H T$ stems from the fact that the return of greedy and $\epsilon$-greedy policies can differ at most $\epsilon H$ in each episode.
        In the next, we aim to estimate the first term in the above equation.
        It involves the greedy policy $\pi^t$ at the $t$-th episode, which leads to a trajectory $\{(s_h^t, a_h^t) \}_{h=1}^H$.
        Note that this trajectory is different from Algorithm~\ref{algo:nnfacounter} that uses the $\epsilon$-greedy policy but we use the same notation on state-action pairs for notational simplicity in this section.
        
        Following \cite{cai2020provably,yang2020function}, we define two quantities $\zeta_{t,h}^1$,  $\zeta_{t,h}^2 \in \mathbb{R}$ for any $h \in [H]$ and $t \in [T]$ based on the greedy policy
	\begin{equation}\label{eq:define_mtg1}
		\begin{split}
			\zeta_{t,h}^1 & := [ V_h^t(s_h^t) - V_{h}^{\pi^t} (s_h^t) \bigr ]  - [ Q_h^t(s_h^t , a_h^t) - Q_{h}^{\pi^t} (s_h^t, a_h^t) ] \,,\\
			\zeta_{t,h}^2 & := [ ( \mathbb{P}_h V_{h+1}^t ) (s_h^t , a_h^t) - (\mathbb{P}_h V_{h+1}^{\pi^t} ) (s_h^t, a_h^t) ] - [ V_{h+1}^t(s_{h+1}^t) - V_{h+1}^{\pi^t} (s_{h+1}^t) ]\,.
		\end{split}
	\end{equation}
	By definition, $\zeta_{t,h}^1$ depends on the randomness of choosing an action $a_h^t \sim \pi_h^t(\cdot | s_h^t)$; and $\zeta_{t,h}^2$ captures the stochastic transition, i.e., the randomness of drawing the next state $s_{h+1}^t$ from $\mathbb{P}_h(\cdot | s_h^t, a_h^t)$.
	Based on the following definition of filtration, $\{{\zeta}_{t,h}^1, \zeta_{t,h}^2 \}$ forms a bounded martingale difference sequence.
	\begin{definition}\cite[Filtration]{cai2020provably}
		For any $(t, h )\in [T] \times [H]$, define $\sigma$-algebras $\mathcal{M}_{t,h,1}$ and  $\mathcal{M}_{t,h,2}$ generated by the following respective state-action sequence as
		\begin{equation}\label{eq:define_filtration} 
			\begin{split}
				\mathcal{M}_{t,h,1} & := \sigma \bigl (	\{(s^\tau_i, a^\tau_i)\}_{(\tau, i)\in [t-1]  \times [H]}  \cup \{(s^t_i, a^t_i)\}_{i\in [h]}  \bigr),  \\
				\mathcal{M}_{t, h ,2} & :=  \sigma \bigl ( \{ (s^\tau_i, a^\tau_i) \}_{(\tau, i)\in [t-1] \times [H]}  \cup \{(s^t_i, a^t_i)\}_{i\in [h]} \cup \{s^t_{h+1}\} \bigr ) \,,    
			\end{split}
		\end{equation}
		where we identify $\mathcal{F}_{t, 0, 2}$ with $\mathcal{M}_{t-1, H, 2}$ for all $t\geqslant 2$ and let $\mathcal{M}_{1,0, 2}$ be the empty set. 
		Further, for any $t\in [T] $, $h \in [H]$ and $m\in [2]$, we define the time-step index $\tau(t, h, m) $ as 
		\begin{equation}\label{eq:time_ordering}
			\tau(t,h,m)=(t-1)\cdot 2H+(h-1)\cdot 2+m\,,
		\end{equation}
		which offers an partial ordering over  the triplets $(t, h , m)\in [T] \times[H]\times[2]$.  
		Moreover, according to Eq.~\eqref{eq:define_filtration}, for any $(t, h, m)$ and $(t', h', m')$ satisfying  $\tau(k,h,m) \leqslant \tau (k',h',m')$, it holds that $\mathcal{M}_{k,h,m}\subseteq \mathcal{M}_{k',h',m'}$. Thus, the sequence of $\sigma$-algebras $\{\mathcal{M}_{t,h,m}\}_{(t,h,m)\in [T] \times[H]\times[2]}$ forms a filtration. 
	\end{definition}

	Accordingly, we have the following regret decomposition result. 
	
	\begin{lemma}[Regret Decomposition \cite{cai2020provably,yang2020function}]\label{lem:regret_decomp}
		Recall the definition of the temporal-difference error $\Gamma_{h}^t \colon \mathcal{S} 	\times \mathcal{A} \rightarrow$ in Eq.~\eqref{eq:td_error} for all 
		$(t,h) \in [T] \times [H]$, then the regret can be decomposed as
		\begin{equation}\label{eq:regret_decomp1}
			\begin{split}
					\text{Regret}(T) & \leqslant  \underbrace{ \sum_{t=1}^T \sum_{h=1}^H\big[ \mathbb{E}_{\pi^{\star}}[ \Gamma^t_h(s_h,a_h)\,|\, s_1 = s^t_1] 
					-  \Gamma^t_h(s^t_h,a^t_h)\big]}_{{\tt Term~(i)}} +    \underbrace{  
					\sum_{t=1}^T \sum_{h=1}^H ( {\zeta}_{t,h} ^1 + \zeta_{t, h }^2 )}_{{\tt Term~(ii)}}  \\
				&   +	\underbrace{\sum_{t=1}^T\sum_{h=1}^H \mathbb{E}_{\pi^{\star}} \big[  \big \langle Q^{t}_h(s_h,\cdot), \pi^{\star}_h(\cdot\,|\, s_h) - \pi_h^t(\cdot| s_h)  \big \rangle_{\mathcal{A}} \big| s_1 = s_1^t \big]}_{{\tt Term~(iii)} \leqslant 0} + \epsilon HT \,,
			\end{split}
		\end{equation}
		where   $\zeta_{t,h}^1$ and $\zeta_{t,h}^2$ are defined in Eq.~\eqref{eq:define_mtg1}. 
	\end{lemma}
	 \begin{proof}
        Remember the definition of the regret, $\tilde{\pi}^t$ is the $\epsilon$-greedy policy and $\pi^t$ is the greedy policy at the $t$-th episode, and then we have
        \begin{equation}\label{eqregret}
        \begin{split}
           \text{Regret}(T) & = \sum_{t=1}^T \bigl [\sind{V}{\star}{1} (s^t_1) - \sind{V}{{\pi}^t}{1} (s^t_1)\bigr] + \sum_{t=1}^T \bigl [\sind{V}{{\pi}^t}{1} (s^t_1) - \sind{V}{\tilde{\pi}^t}{1} (s^t_1)\bigr]  \\
           & \leqslant \sum_{t=1}^T \underbrace{ V_1^{\star} (s_1^t) - V_1^t (s_1^t) }_{(*)} + \sum_{t=1}^T \underbrace{ V_1^t (s_1^t) - V_1^{\pi^t} (s_1^t) }_{(**)} + \epsilon H T \,,
        \end{split}
        \end{equation}
        where the first term (*) can be bounded by \cite{cai2020provably,yang2020function}
        \begin{equation*}
        \begin{split}
            V_1^{\star} (s_1^t) - V_1^t (s_1^t) & = \sum_{h=1}^H\big[ \mathbb{E}_{\pi^{\star}}[ \Gamma^t_h(s_h,a_h)\,|\, s_1 = s^t_1] \\
            & +	\underbrace{\sum_{h=1}^H \mathbb{E}_{\pi^{\star}} \big[  \big \langle Q^{t}_h(s_h,\cdot), \pi^{\star}_h(\cdot\,|\, s_h) - \pi_h^t(\cdot| s_h)  \big \rangle_{\mathcal{A}} \big| s_1 = s_1^t \big]}_{ \leqslant 0}\,,~~ \forall t \in [T]\,,
		\end{split}
        \end{equation*}
        where we use the fact that $\pi^t$ is the greedy policy with respect to $Q_h^t$ for any $(t,h) \in [T] \times [H]$.
        The second term (**) is also bounded by \cite{cai2020provably,yang2020function}
        \begin{equation*}
            V_1^t (s_1^t) - V_1^{\pi^t} (s_1^t) = \sum_{h=1}^H ({\zeta}_{t,h} ^1 + \zeta_{t, h }^2) - \sum_{h=1}^H \Gamma^t_h(s^t_h,a^t_h)\,, \quad \forall t \in [T]\,.
        \end{equation*}
        Finally, we conclude the proof.
    \end{proof}

	In the next, it is natural to employ Azuma-Hoeffding inequality for martingale difference sequences	as below.
	\begin{lemma}\cite[statistical error]{cai2020provably}\label{lemma:bound_mtg}
		For $ {\zeta}_{t,h}^1$ and $\zeta_{t,h}^2$  defined in Eq.~\eqref{eq:define_mtg1} and for any $\delta  \in (0,1)$, with probability at least  $1- \delta $, we have 
		\begin{equation*}
		\sum_{t =1}^T \sum_{h=1}^H ( {\zeta}_{t,h}^1 +\zeta_{t,h}^2)  \lesssim \sqrt{T H^3 \log (2/ \delta )} \,.
		\end{equation*}
	\end{lemma}

	\section{Proofs of transformation on the temporal difference error}
	\label{app:termitd}
		In this section, we aim to transform the temporal difference error in ${\tt Term~(i)}$ to generalization bounds. This is the key part in our proof without bonus function design.
	
	\subsection{TD error under the averaged measure}
Here we build the connection between ${\tt Term~(i)}$ in the regret decomposition and the TD error $\Gamma_{h}^t $ in the $L^2(\mathrm{d} \bar{\mu}^{\tilde{t}}_h)$-integrable space.	

To this end, we need study the relationship between $L^2(\mathrm{d} \mu)$-norm and $L^{\infty}$-norm, where $\mu$ can be any probability measure over $\mathcal{S} \times \mathcal{A}$. 
For any $ f \in L^2(\mathrm{d} \mu)$ with $\delta \leqslant \|f\|_{\infty} $, denote
\begin{equation}\label{eqgset}
    \mathcal{G}_{\delta} := \{ (s,a): |f(s,a)| \geqslant \|f\|_{\infty} - \delta \},~~ \forall (s,a) \in \mathcal{S} \times \mathcal{A}\,,
\end{equation}
then we have the following lemma that $\bar{\mu}^{\tilde{t}}_h(\mathcal{G}_{\delta}) $ can be lower bounded under the $\epsilon$-greedy policy.

\begin{lemma}\label{lemmalowerbound}
    Under the $\epsilon$-greedy policy, considering the set in \cref{eqgset} and the averaged measure $\bar{\mu}^{\tilde{t}}_h$ based on a mini-batch of $\tilde{t}$ historical state-action pairs, we have
    \begin{equation*}
        \bar{\mu}^{\tilde{t}}_h(\mathcal{G}_{\delta}) \geqslant \Omega \left( \Big( \frac{\epsilon}{A} \Big)^H \right),~~ \forall \epsilon \in (0,1)~~\mbox{and}~~ \delta \geqslant 0 \,.
    \end{equation*}
\end{lemma}
{\bf Remark:} Clearly, in the best case, we have $\bar{\mu}^{\tilde{t}}_h(\mathcal{G}_{\delta}) \geqslant \Omega \left( \frac{\epsilon}{A}  \right) $. Accordingly, we denote $K \in [1,H]$ as a MDP-dependent constant to describe the ``myopic'' level of MDPs \cite{dann2022guarantees} such that $\bar{\mu}^{\tilde{t}}_h(\mathcal{G}_{\delta}) \geqslant \Omega \left( ( {\epsilon}/{A} )^K \right)$.

\begin{proof}
For any $ f \in L^2(\mathrm{d} \mu)$ with $\delta \leqslant \|f\|_{\infty} $, we have
    \begin{equation}\label{eqflpinf}
	\| f \|_{L^{2}(\mathrm{d}\mu)} \geqslant \left(  \int_{\mathcal{G}_{\delta}} ( \|f \|_{\infty} - \delta)^2 \mathrm{d}\mu  \right)^{1/2} = ( \| f \|_{\infty} - \delta) [\mu(\mathcal{G}_{\delta})]^{1/2}\,,
\end{equation}
which is also valid to $\bar{\mu}^{\tilde{t}}_h$.
Clearly, $\bar{\mu}^{\tilde{t}}_h(\mathcal{G}_{\delta}) \in [0,1]$.

To prove $\bar{\mu}^{\tilde{t}}_h(\mathcal{G}_{\delta}) > 0$ with the lower bound, 
we consider {\bf the worst case} with $\delta = 0$ and every time step taking non-greedy action with probability $\epsilon$.
That means, we need to find the optimal state-action pair in \cref{eqgset}, which can be achieved by the fact that all state-action pairs have been visited at all time steps. 
It is clear that the cardinality of $\mathcal{G}_{\delta}$ is a non-decreasing function of $\delta$.
Accordingly, there exists $j \in [\tilde{t}]$ such that
\begin{equation*}
    \bar{\mu}^{\tilde{t}}_h(\mathcal{G}_{\delta}) \geqslant \bar{\mu}^{\tilde{t}}_h (\mathcal{G}_{0}) \geqslant \min_{(s,a,h)} \mu_h^{\pi_h^{\tau_j}} (s_h^{\tau_j}, a_h^{\tau_j}) \,,
\end{equation*}
where $\mu_h^{\pi_h^{\tau_j}}$ is the occupancy measure of the policy $\pi^{\tau_j}_h$ at the $h$-step and $t$-th episode.
%For notational simplicity, we omit the subscript $j$ in the above equation, e.g.,  $\mu_h^{\pi_h^{\tau}}$ for short.
Accordingly, $\mu_h^{\pi_h^{\tau_j}}$ admits the following representation
\begin{equation*}
\mu_{h}^{\pi_h^{\tau_j}}\left(s^{\tau_j}_{h}, a^{\tau_j}_{h}\right)=\sum_{s^{\tau_j}_{1}, \ldots, s^{\tau_j}_{h-1}}\left(\prod_{i=1}^{h-1} \sum_{a \in \mathcal{A}} \operatorname{Pr}\left(\pi_h^{\tau} \left(s^{\tau_j}_{i}\right)=a\right) \mathbb{P}_{i}\left(s^{\tau_j}_{i+1} \mid s^{\tau_j}_{i}, a\right)\right) \operatorname{Pr}\left(\pi_h^{\tau_j}\left(s^{\tau_j}_{h}\right)=a_{h}^{\tau_j}\right)\,.
\end{equation*}
Accordingly, in the worst case, at every time step we take any one action with probability $\epsilon/A$ such that
\begin{equation*}
    \mu_{h}^{\pi_h^{\tau_j}}\left(s^{\tau_j}_{h}, a^{\tau_j}_{h}\right) \geqslant \Omega \left( \Big( \frac{\epsilon}{A} \Big)^h \right) \geqslant \Omega \left( \Big( \frac{\epsilon}{A} \Big)^H \right)\,,
\end{equation*}
which implies that
\begin{equation*}
    \bar{\mu}^{\tilde{t}}_h(\mathcal{G}_{\delta}) \geqslant \Omega \left( \Big( \frac{\epsilon}{A} \Big)^H \right)\,,
\end{equation*}
and accordingly we conclude the proof.
\end{proof}

		\begin{lemma}\label{lem:termitd}
		Given a MDP-dependent constant $K \in [1,H]$, for the temporal-difference error $\Gamma_{h}^t $ defined in \cref{eq:td_error} for all 
		$(t,h) \in [T] \times [H]$, under the $\epsilon$-greedy policy, then ${\tt Term~(i)}$ can be upper bounded by 
		\begin{equation*}
			{\tt Term~(i)}  \lesssim \left( \frac{\epsilon}{A} \right)^{-\frac{K}{2}} \sqrt{T} \sum_{h=1}^H \sqrt{\sum_{t=1}^T  \| \Gamma^t_h \|^2_{L^2{(\mathrm{d}\bar{\mu}^{\tilde{t}}_h)}} }  + {\mathcal{O}}(H\sqrt{T}) \,.
		\end{equation*}
	\end{lemma}

	\begin{proof}
		According to the definition of ${\tt Term~(i)} $ in Lemma~\ref{lem:regret_decomp}, we have
		\begin{equation}\label{eq:termi1}
			\begin{split}
				{\tt Term~(i)} 
				& \leqslant \sum_{t=1}^T \sum_{h=1}^H \bigg( \mathbb{E}_{\pi^{\star}} \Big[ \big| \Gamma^t_h(s_h,a_h ) \big|  \Big|  s_1 = s^t_1 \Big] 
				+  \big| \Gamma^t_h(s^t_h,a^t_h) \big| \bigg) \\
				& \leqslant 2 \sum_{t=1}^T \sum_{h=1}^H \| \Gamma^t_h \|_{\infty}  \quad \mbox{[hold for any $(s,a) \in \mathcal{S} \times \mathcal{A}$]}\\
				& \leqslant 2 \sum_{t=1}^T \sum_{h=1}^H \left( \frac{\| \Gamma^t_h \|_{L^2{(\mathrm{d}\bar{\mu}^{\tilde{t}}_h)}}}{\sqrt{\bar{\mu}^{\tilde{t}}_h (\mathcal{G}_{\delta})}} + \delta \right)  \quad \mbox{[taking $\mu := \bar{\mu}^{\tilde{t}}_h$ in \cref{eqflpinf}]} \,.
			\end{split}
		\end{equation}
		Furthermore, by taking $\delta:= t^{-1/2}$ such that $\int_{1}^T t^{-1/2} \mathrm{d} t = \mathcal{O}(\sqrt{T})$, and using $\bar{\mu}^{\tilde{t}}_h(\mathcal{G}_{\delta}) \geqslant \Omega \left( ( {\epsilon}/{A} )^K \right)$ with $K \in [1,H]$ in Lemma~\ref{lemmalowerbound}, the above equation can be further expressed as
			\begin{equation*}
			\begin{split}
				{\tt Term~(i)} & \lesssim \left( \frac{\epsilon}{A} \right)^{-\frac{K}{2}} \sum_{t=1}^T \sum_{h=1}^H   \| \Gamma^t_h \|_{L^2{(\mathrm{d}\bar{\mu}^{\tilde{t}}_h)}}  + {\mathcal{O}}(H\sqrt{T}) \quad \mbox{[using Lemma~\ref{lemmalowerbound}]}\\
				& \leqslant \left( \frac{\epsilon}{A} \right)^{-\frac{K}{2}} \sum_{h=1}^H \sqrt{T} \sqrt{\sum_{t=1}^T  \| \Gamma^t_h \|^2_{L^2{(\mathrm{d}\bar{\mu}^{\tilde{t}}_h)}} }  + {\mathcal{O}}(H\sqrt{T})\,, \quad \mbox{[using elementary inequality]}
			\end{split}
		\end{equation*}
		which concludes the proof.
	\end{proof}

 \subsection{Connection between the TD error and generalization bounds}
 \label{app:tdgen}
 
 Based on Lemma~\ref{lem:termitd}, the key issue left is to bound $\sum_{t=1}^T \| \Gamma^t_h \|^2_{L^2{(\mathrm{d}\bar{\mu}^{\tilde{t}}_h)}} \lesssim o(T)$ for a sublinear regret.
 To this end, we build the connection between $\| \Gamma^t_h \|^2_{L^2{(\mathrm{d}\bar{\mu}^{\tilde{t}}_h)}} $ and generalization bounds.
We first the study the decomposition of $\mathcal{E}^t_h(f)$ in Eq.~\eqref{eqexpecterm} by the following proposition: there exists an extra variance term in the expected risk $\mathcal{E}^t_h(f)$.
 
 	\begin{proposition}\label{prop:exerm}
 	According to the definition of $\mathcal{E}^t_h(f)$ in Eq.~\eqref{eqexpecterm}, then we have
 	\begin{equation}\label{eqbiasvar}
 		\mathcal{E}^t_h(f) = \underbrace{ \| f - \mathbb{T}_h V^t_{h+1} \|^2_{L^2{(\mathrm{d}\bar{\mu}^{\tilde{t}}_h)}} }_{:= \bar{\mathcal{E}}^t_h(f)} + \mathbb{E}_{(s_h, a_h) \sim \bar{\mu}^{\tilde{t}}_h, s_{h+1} \sim \mathbb{P}_h(\cdot|s_h, a_h)} {\tt Var}[V^t_{h+1}(s_{h+1})]\,,
 	\end{equation} 
 	where the variance ${\tt Var}[V^t_{h+1}(s_{h+1})] := \left[ \mathbb{E}_{s_{h+1}}[V^t_{h+1}(s_{h+1})] - V^t_{h+1}(s_{h+1}) \right]^2 $.
 \end{proposition}
 \begin{proof}
 	Denote $s':=s_{h+1}$ for short, we expand $\mathcal{E}^t_h(f)$ as the following expression
 	\begin{equation*}
 		\begin{split}
 			\mathcal{E}^t_h(f) & = \mathbb{E}_{(s_h, a_h) \sim \bar{\mu}^{\tilde{t}}_h, s' } \left[ f(s_h, a_h) - r_h(s_h, a_h) - \mathbb{E}_{s'} V^t_{h+1}(s') + \mathbb{E}_{s'} V^t_{h+1}(s')  - V^t_{h+1}(s') \right]^2\\
 			& = \mathbb{E}_{(s_h, a_h) \sim \bar{\mu}^{\tilde{t}}_h} \left[  f(s_h, a_h) - r_h(s_h, a_h) - \mathbb{E}_{s'} V^t_{h+1}(s') \right]^2 + \mathbb{E}_{(s_h, a_h) \sim \bar{\mu}^{\tilde{t}}_h, s' } {\tt Var} [V^t_{h+1}(s')] \\ 
 			& = \| f - \mathbb{T}_h V^t_{h+1} \|^2_{L^2{(\mathrm{d}\bar{\mu}^{\tilde{t}}_h)}} + \mathbb{E}_{(s_h, a_h) \sim \bar{\mu}^{\tilde{t}}_h, s_{h+1} \sim \mathbb{P}_h(\cdot|s_h, a_h)} {\tt Var}[V^t_{h+1}(s_{h+1})] \,, 
 		\end{split}
 	\end{equation*}
 	where we use $\mathbb{E}_{s' \sim \mathbb{P}_h(\cdot| s_h, a_h)} \left[ \mathbb{E}_{s'}[V^t_{h+1}(s')] - V^t_{h+1}(s') \right] = 0$ and conclude the proof.
 \end{proof}
 
According to the decomposition of $\mathcal{E}_h^t(f)$ \cref{prop:exerm}, $\bar{\mathcal{E}}^t_h(f)$ in Eq.~\eqref{eqbiasvar} is close to the squared Bellman error \cite{duan2021risk}. We are able to transform the estimation of the TD error to generalization error and approximation error as below.
 
 \begin{lemma}\label{lem:tdgen}
 For the temporal-difference error $\Gamma_{h}^t $ defined in Eq.~\eqref{eq:td_error} for all $(t,h) \in [T] \times [H]$, it can be upper bounded in the $L^2{(\mathrm{d}\mu_h^{\tilde{t}})}$ space with
 	\begin{equation*}
 		 \| \Gamma_h^t \|^2_{L^2{(\mathrm{d}\bar{\mu}_h^{\tilde{t}})}} \leqslant \Big[ \mathcal{E}^t_h(\widehat{Q}^t_h) - 	\min_{f \in \mathcal{F}} \mathcal{E}^t_h(f) \Big] + \inf_{f \in \mathcal{F}} \| f - \mathbb{T}_h^{\star} Q^t_{h+1} \|^2_{ L^2{(\mathrm{d} \bar{\mu}^{\tilde{t}}_h) }} \,,
 	\end{equation*}
 	where the first term is the generalization error of $\widehat{Q}^t_h$, the second term is the approximation error in the function class $\mathcal{F}$. 
 \end{lemma}
	\begin{proof}
		According to the definition of the TD error $\Gamma_h^t$ and taking $f:= \widehat{Q}^t_h$ in Eq.~\eqref{eqbiasvar} given by Proposition~\ref{prop:exerm}, we have
		\begin{equation}\label{ehqthhat}
			\begin{split}
				\mathcal{E}^t_h(\widehat{Q}^t_h) & =  \| \widehat{Q}^t_h - \mathbb{T}_h V^t_{h+1} \|^2_{L^2{(\mathrm{d}\bar{\mu}_h^{\tilde{t}})}} + \mathbb{E}_{(s_h, a_h) \sim \bar{\mu}^{\tilde{t}}_h, s_{h+1} \sim \mathbb{P}_h(\cdot|s_h, a_h)} {\tt Var}[V^t_{h+1}(s_{h+1})]\\
				& = \frac{1}{\tilde{t}} \sum_{j=1}^{\tilde{t}} \| \widehat{Q}^t_h - \mathbb{T}_h V^t_{h+1} \|^2_{L^2{(\mathrm{d}\mu_h^{\tau_j})}} + \mathbb{E}_{(s_h, a_h) \sim \bar{\mu}^{\tilde{t}}_h, s_{h+1} \sim \mathbb{P}_h(\cdot|s_h, a_h)} {\tt Var}[V^t_{h+1}(s_{h+1})]\\
				& = \frac{1}{\tilde{t}} \sum_{j=1}^{\tilde{t}} \| \Gamma^t_h \|^2_{L^2{(\mathrm{d}\mu_h^{\tau_j})}} + \mathbb{E}_{(s_h, a_h) \sim \bar{\mu}^{\tilde{t}}_h, s_{h+1} \sim \mathbb{P}_h(\cdot|s_h, a_h)} {\tt Var}[V^t_{h+1}(s_{h+1})]\,,
			\end{split}
		\end{equation} 
		where the second equality holds by the definition of the averaged measure $\bar{\mu}^{\tilde{t}}_h = \frac{1}{\tilde{t}} \sum_{j=1}^{\tilde{t}} \mu_h^{\tau_j} $; and we use $Q_h^t = \widehat{Q}_h^t$ in the last equality as the truncation operation has been given in function classes, see Eqs.~\eqref{eq2nndef} and \eqref{eqdnndef}. 
		Then, taking the infimum on both sides of \cref{eqbiasvar}, we have
		\begin{equation}\label{minehf}
			\begin{split}
				\min_{f \in \mathcal{F}} \mathcal{E}^t_h(f) 
				& =  \inf_{f \in \mathcal{F}} \| f - \mathbb{T}_h V^t_{h+1} \|^2_{L^2{(\mathrm{d}\bar{\mu}^{\tilde{t}}_h)}}  + \mathbb{E}_{(s_h, a_h) \sim \bar{\mu}^{\tilde{t}}_h, s_{h+1} \sim \mathbb{P}_h(\cdot|s_h, a_h)} {\tt Var}[V^t_{h+1}(s_{h+1})] \\
				& = \inf_{f \in \mathcal{F}} \| f - \mathbb{T}_h^{\star} Q^t_{h+1} \|^2_{L^2{(\mathrm{d}\bar{\mu}^{\tilde{t}}_h)}} + \mathbb{E}_{(s_h, a_h) \sim \bar{\mu}^{\tilde{t}}_h, s_{h+1} \sim \mathbb{P}_h(\cdot|s_h, a_h)} {\tt Var}[V^t_{h+1}(s_{h+1})] \,,
			\end{split}
		\end{equation}
	where the second equality holds by $V_{h+1}^t(s_{h+1}) = \max_{a \in \mathcal{A}} Q_{h+1}^t(s_{h+1},a)$. 

Combining Eqs.~\eqref{ehqthhat} and~\eqref{minehf}, we have
\begin{equation}\label{eqtderrest}
	\begin{split}
		\| \Gamma_h^t \|^2_{L^2{(\mathrm{d}\bar{\mu}_h^{\tilde{t}})}}  & = 	\mathcal{E}^t_h(\widehat{Q}^t_h) - 	\min_{f \in \mathcal{F}} \mathcal{E}^t_h(f)  + \inf_{f \in \mathcal{F}} \| f - \mathbb{T}_h^{\star} Q^t_{h+1} \|^2_{L^2{(\mathrm{d}\bar{\mu}^{\tilde{t}}_h)}}\,, 
	\end{split}
\end{equation}
which concludes the proof.
	\end{proof}

Based on Lemma~\ref{lem:tdgen}, we have the following corollary if we consider the approximation error in $L^p(\mathcal{X})$-integrable space, which is needed for our results on deep ReLU neural networks.
\begin{corollary}\label{coro:tdgen}
	Under the same setting of Lemma~\ref{lem:tdgen}, we have
	 	\begin{equation*}
	\| \Gamma_h^t \|^2_{L^2{(\mathrm{d}\bar{\mu}_h^{\tilde{t}})}}  \lesssim \Big[ \mathcal{E}^t_h(\widehat{Q}^t_h) - 	\min_{f \in \mathcal{F}} \mathcal{E}^t_h(f) \Big] + \inf_{f \in \mathcal{F}} \| f - \mathbb{T}_h^{\star} Q^t_{h+1} \|^2_{L^4{(\mathcal{X}})} \,.
	\end{equation*}
\end{corollary}
\begin{proof}
Following the proof of \cref{lem:tdgen}, this result can be easily obtained by Cauchy-Schwartz inequality.
To be specific, for any probability measure $\mu$, we have
\begin{equation*}
	\| f \|_{L^p(\mathrm{d} \mu)} \leqslant \| f \|_{L^{2p}(\mathcal{X})} \left( \int_{\mathcal{X}} |g(\bm x)|^2 \mathrm{d} \bm x \right)^{\frac{1}{2p}}  \lesssim \| f \|_{L^{2p}(\mathcal{X})} \,,
\end{equation*}
where $g$ is the probability density function associated with the probability measure $\mu$. Note that the result here still holds true for the approximation error in $L^{\infty}(\mathcal{X})$ if we use H\"{o}lder inequality, but this condition is much stronger as it requires the target Q function to be continuous.
\end{proof}

\section{Generalization bounds on non-iid data}
\label{app:genealization}

In this section, we prove that the traditional Rademacher complexity is still valid for independent but non-identically distributed data under a well-defined measure.
Similarly, such result is also valid to local Rademacher complexity.
The key fact is that, the classical Rademacher complexity \cite{mohri2018foundations} is still valid as McDiarmid’s bound only requires the independent property.

For description simplicity, we consider a general setting beyond our reinforcement learning task, i.e., learning with $n$ independent but non-identical distributed data $X = \{ \bm x_i \}_{i=1}^n$ in $\mathbb{R}^d$ with $\bm x_i \sim \mu_i, \forall i \in [n]$. Define the average measure $\bar{\mu} := \frac{1}{n} \sum_{i=1}^n \mu_i$, we have
\begin{equation}\label{eqaverage}
	\mathbb{E}_{\bm x \sim \bar{\mu}} [f(\bm x)] = \frac{1}{n} \sum_{i=1}^n \int_{\mathbb{R}^d} f(\bm x) \mathrm{d} \mu_{i}(\bm x) = \frac{1}{n} \sum_{i=1}^n \mathbb{E}_{\bm x \sim \mu_i} [f(\bm x)]\,.
\end{equation} 
Accordingly, the \emph{empirical Rademacher complexity} of a function class $\mathcal{F}$ on the sample set $X$ is defined as
\begin{equation}\label{empiricalrc}
	\widehat{\mathcal{R}}_n (\mathcal{F}, X) = \frac{1}{n} \mathbb{E}_{\bm \xi} \left[ \sup_{f \in \mathcal{F}} \sum_{i=1}^n \xi_i f(\bm x_i) \right]\,,
\end{equation} 
where the expectation is taken over $\bm \xi = \{\xi_1, \xi_2, \cdots, \xi_n\}$, i.e., Rademacher random variables, with $\mathrm{Pr}(\xi_i=1) = \mathrm{Pr}(\xi_i=-1) = 1/2$.
The related \emph{Rademacher complexity} under our non-iid setting is defined as

\begin{equation*}
	\mathcal{R}_{n}(\mathcal{F})=\mathbb{E}_{\bm x_1 \sim \mu_1, \cdots, \bm x_n \sim \mu_n}\left[\frac{1}{n} \mathbb{E}_{\bm \xi}\left[\sup _{f \in \mathcal{F}} \sum_{i=1}^{n} \xi_{i} f\left(\bm x_{i}\right)\right]\right]\,,
\end{equation*}
where the expectation is taken over $\{ \bm x_i \}_{i=1}^n$ with respect to each probability measure $\{ \mu_i \}_{i=1}^n$.
This definition follows the classical \emph{Rademacher complexity} \cite{mohri2018foundations} on iid samples to intuitively indicates how expressive the function class is.
Besides, in our proof, we also need a notation of \emph{local Rademacher complexity} on a set of vectors, where ``local'' means that the class over which the Rademacher process is defined is a subset of the original class. 
Following the same style with Rademacher complexity, the local Rademacher complexity under the non-iid setting is defined as $ \mathcal{R}_n\{ f \in \mathcal{F}: \mathbb{E}_{\bar{\mu}}f^2 \leqslant R \}$, and the empirical local Rademacher complexity is defined as $  \widehat{\mathcal{R}}_n \{ f \in \mathcal{F}: P_nf^2 \leqslant R \}$, where we denote $P_n f := \frac{1}{n} \sum_{i=1}^n f(\bm x_i)$ for short.

Besides, Rademacher complexity is also related to covering number, a metric for estimation of a hypothesis space. Here we give the definition of covering number, that is also used in this work.
\begin{definition} \cite[Definition 5.1, covering number]{wainwright2019high}
	Let $(\mathcal{F}, \| \cdot \|)$ be a norm space. A $\delta$-cover of the set $\mathcal{F}$ with respect to $\| \cdot \|$ is a set $\{ \theta_1, \cdots, \theta_n \} \subseteq \mathcal{F}$ such that for each $\theta \in \mathcal{F}$, there exists some $i \in [n]$ such that $\| \theta - \theta_i \| \leqslant \delta$. The $\delta$-covering number $\mathscr{N}(\delta, \mathcal{F}, \| \cdot \|)$ is the cardinality of the minimal $\delta$-cover.
\end{definition}
In this work, we consider the covering number with two types of norms, one is $\mathscr{N}(\epsilon, \mathcal{F}, \| \cdot \|_{\infty}) $ and the other is
$\mathscr{N}(\epsilon, \mathcal{F}, \| \cdot \|_2) := \sup_n \sup_{P_n}  \mathscr{N}(\epsilon, \mathcal{F}, \| \cdot \|_{L_2(P_n)})$ \cite{mendelson2002improving}.

\subsection{Rademacher complexity on non-iid data}
\label{app:radenoniid}

Based on the definition of Rademacher complexity and its empirical version, we have the following lemma. 
\begin{lemma}\label{lem:EfRn}
	Let $X = \{ \bm x_i \}_{i=1}^n$ be an independent but non-identical distributed data set with $\bm x_i \sim \mu_i, \forall i \in [n]$, and $R_n(\mathcal{F})$ be the Rademacher complexity of the function class $\mathcal{F}$ on $X$, denote the averaged probability measure as $\bar{\mu} := \frac{1}{n} \sum_{i=1}^n \mu_i$, then we have
	\begin{equation*}
		\mathbb{E}_{\bm x_1 \sim \mu_1, \cdots, \bm x_n \sim \mu_n} \left[\sup _{f \in \mathcal{F}}\left(\mathbb{E}_{\bm x \sim \bar{\mu}}[f(\bm x)]-\frac{1}{n} \sum_{i=1}^{n} f\left(\bm x_{i}\right)\right)\right] \leq 2 \mathcal{R}_{n}(\mathcal{F})\,.
	\end{equation*}
\end{lemma}

\begin{proof}
	The proof follows with the classical Rademacher complexity \cite[Chapter 26]{shalev2014understanding} apart from the averaged measure. %Here we just include it for self-completeness.
	Take a copy of $X$, i.e., $X'=\{ \bm x'_i \}_{i=1}^n$ such that $X'$ is independent but $\bm x'_i \sim \mu_i, \forall i \in [n]$. According to Eq.~\eqref{eqaverage}, we have
	\begin{equation}\label{eqsymmcopy}
		\mathbb{E}_{\bm x \sim \bar{\mu}} [f(\bm x)] = \mathbb{E}_{\bm x'_1 \sim \mu_1, \cdots, \bm x'_n \sim \mu_n} \left[ \frac{1}{n} \sum_{i=1}^n f(\bm x_i') \right] \,.
	\end{equation}
	Note that every possible configuration/value of $\bm \xi$ has probability of $1/2^n$ due to $\bm \xi \in \{ -1, 1 \}^n$. Without loss of generality, we can permute any configuration of $\bm \xi$ of such that
	\begin{equation*}
		\xi_{u_1} = \xi_{u_2} = \cdots = \xi_{u_k} = 1, \quad \xi_{u_{k+1}} = \xi_{u_{k+2}} = \cdots = \xi_{u_n} = -1,~~ k \in \{ 0\} \cup [n]\,,
	\end{equation*}
	where $\bm u = \{ u_1, u_2, \cdots, u_n \}$ is a permutation of $\{ 1,2, \dots, n \}$.
	Accordingly, for any configuration of $\bm \xi$, we have
	\begin{equation}\label{eqsymmetric}
		\begin{aligned}
			\mathbb{E}_{\{\bm x_i \}_{i=1}^n} & {\left[\mathbb{E}_{\{\bm x'_i \}_{i=1}^n} \left[\sup _{f \in \mathcal{F}}\left(\frac{1}{n} \sum_{i=1}^{n} \xi_{i}\left(f\left(\bm x_{i}^{\prime}\right)-f\left(\bm x_{i}\right)\right)\right)\right]\right] } \\
			&=\mathbb{E}_{\{\bm x_i \}_{i=1}^n} \left[\mathbb{E}_{\{\bm x'_i \}_{i=1}^n} \left[\sup _{f \in \mathcal{F}}\left(\frac{1}{n}\left(\sum_{i=1}^{k}\left(f\left(\bm x_{u_{i}}^{\prime}\right)-f\left(\bm x_{u_{i}}\right)\right)+\sum_{i=k+1}^{n}\left(f\left(\bm x_{u_{i}}\right)-f\left(\bm x_{u_{i}}^{\prime}\right)\right)\right)\right)\right]\right] \\
			&=\mathbb{E}_{\{\bm x_i \}_{i=1}^n} \left[ \mathbb{E}_{\{\bm x'_i \}_{i=1}^n} \left[\sup _{f \in \mathcal{F}}\left(\frac{1}{n} \sum_{i=1}^{n}\left(f\left(\bm x_{i}^{\prime}\right)-f\left(\bm x_{i}\right)\right)\right)\right]\right] \,,
		\end{aligned}
	\end{equation}
	where we use the fact that $\bm x_{u_i}$ and $\bm x'_{u_i}$ are independent and symmetric.
	Based on this, we obtain
	\begin{equation}
		\begin{aligned}
			\mathbb{E}_{\{\bm x_i \}_{i=1}^n} & {\left[\sup _{f \in \mathcal{F}}\left(\mathbb{E}_{\bm x \sim 
					\bar{\mu}}[f(\bm x)]-\frac{1}{n} \sum_{i=1}^{n} f\left(\bm x_{i}\right)\right)\right] } \\
			&=\mathbb{E}_{\{\bm x_i \}_{i=1}^n} \left[\sup _{f \in \mathcal{F}}\left(\mathbb{E}_{\{\bm x'_i \}_{i=1}^n} \left[\frac{1}{n} \sum_{i=1}^{n} f\left(\bm x_{i}^{\prime}\right)\right]-\frac{1}{n} \sum_{i=1}^{n} f\left(\bm x_{i}\right)\right)\right] \quad \text { [using Eq.~\eqref{eqsymmcopy}] } \\
			&=\mathbb{E}_{\{\bm x_i \}_{i=1}^n} \left[\sup _{f \in \mathcal{F}}\left(\mathbb{E}_{\{\bm x'_i \}_{i=1}^n} \left[\frac{1}{n} \sum_{i=1}^{n} f\left(\bm x_{i}^{\prime}\right)-\frac{1}{n} \sum_{i=1}^{n} f\left(\bm x_{i}\right)\right]\right)\right] \\
			& \leqslant \mathbb{E}_{\{\bm x_i \}_{i=1}^n} \left[\mathbb{E}_{\{\bm x'_i \}_{i=1}^n} \left[\sup _{f \in \mathcal{F}}\left(\frac{1}{n} \sum_{i=1}^{n} f\left(\bm x_{i}^{\prime}\right)-\frac{1}{n} \sum_{i=1}^{n} f\left(\bm x_{i}\right)\right)\right]\right] \quad \text { [Jensen's inequality] } \\
			&=\mathbb{E}_{\{\bm x_i \}_{i=1}^n} \left[\mathbb{E}_{\{\bm x'_i \}_{i=1}^n} \left[\mathbb{E}_{\bm{\xi}}\left[\sup _{f \in \mathcal{F}}\left(\frac{1}{n} \sum_{i=1}^{n} \xi_{i}\left(f\left(\bm x_{i}^{\prime}\right)-f\left(\bm x_{i}\right)\right)\right)\right]\right]\right] \quad \text { [using Eq.~\eqref{eqsymmetric}] } \\
			& \leqslant \mathbb{E}_{\{\bm x_i \}_{i=1}^n} \left[\mathbb{E}_{\bm{\xi}}\left[\sup _{f \in \mathcal{F}}\left(\frac{1}{n} \sum_{i=1}^{n} \xi_{i} f\left(\bm x_{i} \right)\right)\right]\right]+\mathbb{E}_{\{\bm x'_i \}_{i=1}^n} \left[\mathbb{E}_{\bm{\xi}}\left[\sup_{f \in \mathcal{F}}\left(\frac{1}{n} \sum_{i=1}^{n} \xi_{i} f\left(\bm x'_{i}\right)\right)\right]\right] \\
			&=2 \mathcal{R}_{n}(\mathcal{F}) \,,
		\end{aligned}
	\end{equation}
	where the last inequality holds by the fact that $\xi_i$ and $-\xi_i,~ i \in [n]$ admit the same distribution, and multiplying each term in the summation by a Rademacher variable $\xi_i$ will not change the expectation due to $\mathbb{E} {\xi_i} = 0$.  
	
\end{proof}

Based on the above lemma, we demonstrate that the Rademacher complexity can be well approximated by the empirical Rademacher complexity under our non-iid setting.

\begin{lemma}\label{lem:expRn}
	Under the same setting of~\cref{lem:EfRn}, for any $f \in \mathcal{F}$, assume that $|f(\bm x) - f(\bm x')| \leqslant c,~ \forall \bm x, \bm x' \in \mathrm{dom}(f)$ for some constant $c > 0$, for any $\delta \in (0,1)$, the following proposition holds with probability at least $1 - \delta$ 
	\begin{equation}\label{EfRnhat}
		\mathrm{Pr} \left(\mathbb{E}_{\bm x \sim \bar{\mu}}(f(\bm x))-\frac{1}{n} \sum_{i=1}^{n} f\left(\bm x_{i}\right) \geqslant 2 \widehat{\mathcal{R}}_{n}(\mathcal{F}, X)+3\delta \right) \leqslant 2\exp\left(-\frac{2n\delta^2}{c^2}\right)\,.
	\end{equation}
\end{lemma}
\begin{proof}
	The proof follows with the classical Rademacher complexity \cite[Chapter 26]{shalev2014understanding} apart from the averaged measure. %Here we just include it for self-completeness.
	Recall the definition of the empirical Rademacher complexity in Eq.~\eqref{empiricalrc}, $\widehat{\mathcal{R}}_n(\mathcal{F}, X)$ is a function of $n$ random variables $\{ \bm x_i \}_{i=1}^n$. Moreover, due to $|f(\bm x) - f(\bm x')|\leqslant c$, $\widehat{\mathcal{R}}_n(\mathcal{F}, X)$ satisfies the precondition for McDiarmid's inequality by at most $c/n$, which only requires independence of random variables without the identically distributed condition 
	\begin{equation*}
		\mathrm{Pr}\left(   \widehat{\mathcal{R}}_n(\mathcal{F}, X) - \mathbb{E}_{\{ \bm x_i \}_{i=1}^n} [ \widehat{\mathcal{R}}_n(\mathcal{F}, X)]  \geqslant \delta  \right) \leqslant \exp\left(-\frac{2n\delta^2}{c^2}\right)\,,
	\end{equation*}
	which implies
	\begin{equation}\label{RnhatRn}
		\mathrm{Pr}\left(  \left| \widehat{\mathcal{R}}_n(\mathcal{F}, X) - {\mathcal{R}}_n(\mathcal{F}) \right| \geqslant \delta  \right) \leqslant 2\exp\left(-\frac{2n\delta^2}{c^2}\right)\,.
	\end{equation}
	
	By Lemma~\ref{lem:EfRn}, we have
	\begin{equation*}
		\mathbb{E}_{\{ \bm x_i \}_{i=1}^n} \left[\mathbb{E}_{\bm x \sim \bar{\mu}}[f(\bm x)]-\frac{1}{n} \sum_{i=1}^{n} f\left(\bm x_{i}\right)\right] \leqslant	\mathbb{E}_{\{ \bm x_i \}_{i=1}^n} \left[\sup _{f \in \mathcal{F}}\left(\mathbb{E}_{\bm x \sim \bar{\mu}}[f(\bm x)]-\frac{1}{n} \sum_{i=1}^{n} f\left(\bm x_{i}\right)\right)\right] \leqslant 2 \mathcal{R}_{n}(\mathcal{F})\,.
	\end{equation*}
	Denote event ${\tt A}$ as
	\begin{equation*}
		\left[\mathbb{E}_{\bm x \sim \bar{\mu}}[f(\bm x)]-\frac{1}{n} \sum_{i=1}^{n} f\left(\bm x_{i}\right)\right] - \mathbb{E}_{\{ \bm x_i \}_{i=1}^n} \left[\mathbb{E}_{\bm x \sim \bar{\mu}}[f(\bm x)]-\frac{1}{n} \sum_{i=1}^{n} f\left(\bm x_{i}\right)\right] \geqslant \delta\,,
	\end{equation*}
	we use McDiarmid's inequality again to obtain $\mathrm{Pr}({\tt A}) \leqslant e^{-2n\delta^2/c^2}$ since $\mathbb{E}_{\bm x \sim \bar{\mu}}[f(\bm x)]-\frac{1}{n} \sum_{i=1}^{n} f\left(\bm x_{i}\right)$ can be regarded as a function of $\{ \bm x_i \}_{i=1}^n$ and any variations of $\{ \bm x_i \}_{i=1}^n$ would change the outcome by at most $c/n$.
	Denote event ${\tt B}$ as $	\mathbb{E}_{\bm x \sim \bar{\mu}}[f(\bm x)]-\frac{1}{n} \sum_{i=1}^{n} f\left(\bm x_{i}\right) - 2 \mathcal{R}_n(\mathcal{F}) \geqslant \delta$, we have $\mathrm{Pr}({\tt B}) \leqslant \mathrm{Pr}({\tt A}) \leqslant e^{-2n\delta^2/c^2}$.
	
	Further, denote the event ${\tt C}$ as $	\widehat{\mathcal{R}}_n(\mathcal{F},X) \geqslant \mathcal{R}_n(\mathcal{F}) - \delta $, we have $\mathrm{Pr}({\tt C}) \geqslant 1 - \exp(-2n\delta^2/c^2)$ by \cref{RnhatRn}. Denote the event ${\tt D}$ as $	\mathbb{E}_{\bm x \sim \bar{\mu}}(f(\bm x))-\frac{1}{n} \sum_{i=1}^{n} f\left(\bm x_{i}\right) \geqslant 2 \widehat{\mathcal{R}}_n(\mathcal{F}) + 3\delta $, we have
	\begin{equation*}
		\begin{aligned}
			\mathrm{Pr}\left( 	\mathbb{E}_{\bm z \sim \bar{\mu}}(f(\bm x))-\frac{1}{n} \sum_{i=1}^{n} f\left(\bm x_{i}\right) \geqslant 2 \widehat{\mathcal{R}}_n(\mathcal{F}) + 3\delta \right) &=\mathrm{Pr}({\tt D}) =\mathrm{Pr}({\tt C} \cap {\tt D})+\mathrm{Pr}({\tt C} \cup {\tt D})-\mathrm{Pr}({\tt C}) \\
			& \leqslant \mathrm{Pr}({\tt B})+1-\mathrm{Pr}({\tt C}) \\
			&=2 \exp\left( -2 n \delta^{2} / c^{2}\right)\,,
		\end{aligned}
	\end{equation*}
	which concludes the proof.
	
\end{proof}

Similar to the proof of Lemma~\ref{lem:expRn}, it is easy to verify that, the standard Massart's lemma and the Talagrand’s Contraction Lemma (empirical Rademacher complexity of Lipschitz function class) in \cite[Chapter 26]{shalev2014understanding} are valid to our independent but non-iid setting.

\subsection{Local Rademacher complexity}
Here we present some results on local Rademacher complexity \cite{bartlett2005local} that is needed in this work.
The used lemmas here are still valid for our independent but non-identically distributed data.
Since the proof framework is similar to what we present for Rademacher complexity, we omit the proofs here.

When applying local Rademacher complexity, we need the following definition.
\begin{definition}
	A function $\psi: \mathbb{R}_+ \rightarrow \mathbb{R}_+$ is sub-root if it is non-negative, non-decreasing, and if $\psi(x)/\sqrt{x}$ is non-increasing.
\end{definition}

\begin{lemma}\cite[Theorem 2]{lei2016local}\label{lem:empirical-local-rademacher-L2}
	Let $\mathcal{F}$ be a function class with $\| f \|_{\infty} \leqslant b,~\forall f \in \mathcal{F}$ and $\widetilde{F}:= \{ f - g: f,g \in \mathcal{F} \}$, and $P_n f := \frac{1}{n} \sum_{i=1}^n f(\bm x_i)$, then taking the average measure $\bar{\mu}$, we have
	\begin{equation*}
		\begin{split}
			\mathcal{R}_n\{f\in \mathcal{F} :\mathbb{E}_{\bar{\mu}}f^2 \leqslant R \} & \leqslant \inf_{\epsilon>0}\Bigg[2 \mathcal{R}_n\{f\in \widetilde{\mathcal{F}}: P_nf^2 \leqslant  \epsilon^2\}+ \frac{ 8b\log \mathscr(\epsilon/2,\mathcal{F},\|\cdot\|_2)}{n}\\ & \qquad +\sqrt{\frac{2R\log \mathscr{N}(\epsilon/2, \mathcal{F},\|\cdot\|_2)}{n}}\Bigg]\,.
		\end{split}
	\end{equation*}
	
\end{lemma}

\begin{lemma}\label{lemma:lrade}\cite[Theorem 3.3, modified version]{bartlett2005local}
	Let $f$ be a class of functions with ranges in $[a,b]$ and assume that there exists some functional $T: \mathcal{F} \rightarrow \mathbb{R}^+$ and some constant $B$ such that ${\tt Var}(f) \leqslant T(f) \leqslant BPf$ for every $f \in \mathcal{F}$. 
	Let $P_n$ be the empirical measure supported on the independent data points $\{ \bm x_i \}_{i=1}^n$ with the averaged measure $\bar{\mu} := \frac{1}{n} \mu_i$,
	Let $\psi$ be a sub-root function with the fixed point $R^*$. If for any $R \geqslant R^*$, $\psi$ satisfies
	\begin{equation*}
		\psi(R) \geqslant B \mathcal{R}_n \{ f \in \mathcal{F}: T(f) \leqslant R \}\,,
	\end{equation*}
	then for any $J > 1$ and $\delta \in (0,1)$, with probability at least $1 - \delta$, we have
	\begin{equation*}
		\mathbb{E}_{\bar{\mu}}f \leqslant \frac{J}{J-1} P_n f + \frac{c_1J}{B}R^* + (c_2(b-a) + c_3 BJ) \frac{\log(1/\delta)}{n}\,,
	\end{equation*}
	where $c_1$, $c_2$, $c_3$ are some positive constants.
\end{lemma}

\begin{lemma}\cite[Refined entropy integral]{mendelson2002improving} \label{lem:refined-dudley}
	Let $P_n$ be the empirical measure supported on the independent data points $\{ \bm x_i \}_{i=1}^n$. For any function class $\mathcal{F}$ and any monotone sequence $\{ \epsilon_k\}_{k=0}^\infty$ decreasing to $0$ such that $\epsilon_0 \geqslant \sup_{f\in \mathcal{F}}\sqrt{P_nf^2}$, the following inequality holds for every non-negative integer $N$
	\begin{equation}\label{refined-dudley-integral}
		\widehat{\mathcal{R}}_n (\mathcal{F},X) \leqslant 4\sum_{k=1}^N\epsilon_{k-1}\sqrt{\frac{\log \mathscr{N}(\epsilon_k, \mathcal{F},\|\cdot\|_{2})}{n}}+\epsilon_N\,.
	\end{equation}
\end{lemma}

\section{Proofs of regret bounds via deep ReLU neural networks}
\label{app:dnn}

In this section, we give the proofs of regret bounds via deep ReLU neural networks according to the function class of $\mathbb{T}_h^{\star} Q$ in Besov spaces.

To conclude our proof, we need the following lemma that how well the functions in the Besov space can be approximated by deep neural networks with the ReLU activation.
Here the approximation error is defined in the $L^4(\mathcal{X})$-integrable space (\emph{c.f.} Corollary~\ref{coro:tdgen}).
\begin{lemma}\label{lem:appbesov}
	(Approximation error in Besov space) \cite[Proposition 1, modified version]{suzuki2019adaptivity}
	Assume that the smoothness parameter $\alpha$ satisfies
	\begin{equation*}
		\alpha > \eta := d(1/p - 1/4)_+\,,
	\end{equation*}
	then there exists a deep neural network architecture $\mathcal{F}_{\tt DNN}(L,m,S,B)$ with $\nu := (\alpha - \eta) / (2\eta)$ and a large $N$ such that
	\begin{equation}\label{eq:lsmbbes}
		L \asymp \log N, ~~S \asymp N, ~~m \asymp N \log N,~~ \mbox{and}~ B \asymp N^{1/\nu + 1/d}\,,
	\end{equation}
then it holds that
	\begin{equation*}
		\sup_{f^* \in \mathcal{B}^{\alpha}_{p,q}(\mathcal{X})} \inf_{f \in \mathcal{F}_{\tt DNN}(L,m,S,B)} \| f - f^* \|_{L^4(\mathcal{X})} \lesssim N^{-\frac{\alpha}{d}}\,, \qquad \forall q >0\,.
	\end{equation*}
\end{lemma}

In our proof, we need the following result on local Rademacher complexity of deep ReLU neural networks.
\begin{lemma}\label{lem:lradebes}
	Let $X = \{ \bm x_i \}_{i=1}^n \subseteq [0,1]^d$ be an independent but non-identical distributed data set with $\bm x_i \sim \mu_i, \forall i \in [n]$, and $	\mathcal{R}_n\{ f \in \mathcal{F}_{\tt DNN}: Pf^2 \leqslant R \}$ be the local Rademacher complexity of the function class $\mathcal{F}_{\tt DNN}$ on $X$ defined in Eq.~\eqref{eqdnndef}, denote the averaged measure as $\bar{\mu} := \frac{1}{n} \sum_{i=1}^n \mu_i$, then for a large $N$, we have
	\begin{equation}
	\mathcal{R}_n\{ f \in \mathcal{F}_{\tt DNN}: \mathbb{E}_{\bar{\mu}}f^2 \leqslant R \} \lesssim \left( \frac{1}{n} +  \sqrt{\frac{R}{n}} \right) \sqrt{N[(\log N)^2 + \log n]} + \frac{H N [(\log N)^2 + \log n]}{n} \,.
	\end{equation}
\end{lemma}
{\bf Remark:} The parameter $N$ depends on the number of the training data $n$, but it will be determined later.
\begin{proof}
	According to \cite[Lemma 3]{suzuki2019adaptivity}, the covering number of $\mathcal{F}_{\tt DNN}$ can be bounded by 
	\begin{align*}
		\log \mathscr{N}(\varepsilon,\mathcal{F}_{\tt DNN},\|\cdot\|_2) \leqslant \log \mathscr{N}(\varepsilon,\mathcal{F}_{\tt DNN},\|\cdot\|_\infty) 
		& \leq 2 S L  \log \left( \frac{L (B \vee 1)  (m+1)}{\varepsilon} \right)  \\
		& \lesssim N \left[(\log N)^2 + \log \left( \frac{1}{\varepsilon} \right) \right]\,.
	\end{align*}
	
	Denote $\widetilde{\mathcal{F}}_{\tt DNN} = \{ f - g: f, g \in \mathcal{F}_{\tt DNN} \}$, it satisfies 
	\begin{equation}\label{eqcntf}
		\begin{split}
			\log \mathscr{N}(\varepsilon,\widetilde{\mathcal{F}}_{\tt DNN},\|\cdot\|_2) \leqslant 2 \log \mathscr{N}\left(\frac{\varepsilon}{2},\mathcal{F}_{\tt DNN},\|\cdot\|_2 \right) 
			& \leqslant 2 \log \mathscr{N}\left(\frac{\varepsilon}{2},\mathcal{F}_{\tt DNN},\|\cdot\|_{\infty} \right)   \\
			& \lesssim N \left[(\log N)^2 + \log \left( \frac{2}{\varepsilon} \right) \right]\,.
		\end{split}
	\end{equation}

	According to Lemma~\ref{lem:refined-dudley}, taking $\varepsilon_j = 2^{-j} \varepsilon$, and using the inequality
	\begin{equation*}
		\mathscr{N}(\varepsilon_j,\{f\in \widetilde{\mathcal{F}}_{\tt DNN}: P_nf^2 \leqslant \varepsilon^2\},\|\cdot\|_{2}) \leqslant \mathscr{N}(\varepsilon_j/2, \widetilde{\mathcal{F}}_{\tt DNN},\|\cdot\|_{2})\,,
	\end{equation*}
	 then the following inequality holds for any $J \in\mathbb{N}^+$:
	\begin{equation}\label{cor-local-rademacher-1-3}
		\begin{split}
			 \mathcal{R}_n\{f\in \widetilde{\mathcal{F}}_{\tt DNN} :P_nf^2 \leqslant \varepsilon^2\}&=	\mathbb{E} 	\widehat{ \mathcal{R}}_n\{f\!\in\!\widetilde{\mathcal{F}}_{\tt DNN}:P_nf^2  \leqslant \varepsilon^2\} \\
			 & \leqslant 4 \mathbb{E}  \sum_{j=1}^J \varepsilon_{j-1}\sqrt{\frac{\log\mathscr{N}(\varepsilon_j/2,\widetilde{\mathcal{F}}_{\tt DNN},\|\cdot\|_{2})}{n}}\!+\!\varepsilon_J\\
			&\leqslant 4 \mathbb{E}  \sum_{j=1}^J 2^{-(j-1)}\varepsilon \sqrt{\frac{2\log\mathscr{N}( \frac{\varepsilon}{2^{(j+1)}} ,\widetilde{\mathcal{F}}_{\tt DNN},\|\cdot\|_{\infty})}{n}} + \varepsilon_J \\
			& \lesssim \frac{\varepsilon}{\sqrt{n}} \sum_{j=1}^J   2^{-(j-1)}  \sqrt{2N\left[(\log N)^2 + \log \left( \frac{2^{j+1}}{\varepsilon} \right) \right]} + 2^{-J}\varepsilon \\
			& \lesssim \frac{\varepsilon}{\sqrt{n}} \sqrt{N\left[(\log N)^2 + \log \left( \frac{1}{\varepsilon} \right) \right]}\,, \quad \mbox{[taking $J \rightarrow \infty$]}
		\end{split}
	\end{equation}
	where the first inequality holds by Lemma~\ref{lem:refined-dudley} and the second and third inequalities hold by Eq.~\eqref{eqcntf}.
	The last inequality uses the fact that $\sum_{j=0}^{\infty} \frac{\sqrt{j+1}}{2^{j-1}} < \infty$. 
	
	According to Lemma~\ref{lem:empirical-local-rademacher-L2} with $\sup_{f \in \mathcal{F}_{\tt DNN}} \| f \|_{\infty} \leqslant H$, we have
	\begin{equation}\label{lradel2}
		\begin{split}
		\mathcal{R}_n\{f\in \mathcal{F}_{\tt DNN} :Pf^2 \leqslant  R\} & \lesssim \inf_{\varepsilon>0}\Bigg[2\mathbb{E} \mathcal{R}_n \{f\in \widetilde{\mathcal{F}}_{\tt DNN} :P_nf^2 \leqslant \varepsilon^2\} \\
			& + \frac{8H  N \left[(\log N)^2 + \log \left( \frac{1}{\varepsilon} \right) \right]}{n}+\sqrt{\frac{2r N \left[\log^2 N + \log \left( \frac{1}{\varepsilon} \right) \right]}{n}}\Bigg]\\
			& \lesssim \inf_{\varepsilon>0}\Bigg[ \frac{\epsilon + \sqrt{2R}}{\sqrt{n}} \sqrt{N \left[(\log N)^2 + \log \left( \frac{1}{\varepsilon} \right) \right]} + \frac{H  N \left[\log^2 N + \log \left( \frac{1}{\varepsilon} \right) \right]}{n} \Bigg]\\
			& \lesssim \frac{n^{-1/2} + \sqrt{R}}{\sqrt{n}} \sqrt{N \left(\log^2 N + \log n \right)} + \frac{H  N \left(\log^2 N + \log n \right)}{n} := \psi(R)\,,
		\end{split}
	\end{equation}
	where we choose $\varepsilon := n^{-1/2}$ in the last inequality, and then we conclude the proof.	
\end{proof}

Based on the above result, we have the following proposition on generalization bounds in Besov spaces under non-iid state-action pairs.
\begin{proposition}\label{prop:genbes}
	Given the solution $\widehat{Q}_h^t = \argmin_{f \in \mathcal{F}_{\tt DNN} } \widehat{\mathcal{E}}_h^t(f)$ in Eq.~\eqref{eqerm}, then for a large $N$ and any $\delta \in (0,1)$, with probability at least $1 - \delta$, we have
	\begin{equation*}
		\mathcal{E}^t_h(\widehat{Q}^t_h) - 	\min_{f \in \mathcal{F}_{\tt DNN}} \mathcal{E}^t_h(f) \lesssim  \frac{N \left[(\log N)^2 + \log n \right]}{n} + \frac{H \sqrt{N\left[(\log N)^2 + \log n \right]}}{n} + \frac{H^2 \log(1/\delta)}{n}\,,
	\end{equation*}
where $n:=\tilde{t}$ in our RL setting and $N$ depends on $t$ which needs further determined.
\end{proposition}
\begin{proof}
	It is clear that $\psi(R)$ defined in Eq.~\eqref{lradel2} in Lemma~\ref{lem:lradebes} is a sub-root function. Therefore, the fixed point $R^*$ of $\psi(R)$ can be analytically solved by the equation $R^*=\psi(R^*)$, which leads to
	\begin{equation*}
		R^* \lesssim \frac{\sqrt{N \left[(\log N)^2 + \log n \right]}}{n} + \frac{H  N \left[(\log N)^2 + \log n \right]}{n}\,.
	\end{equation*}
	Strictly speaking, there is an extra term ${{N \left[(\log N)^2 + \log n \right]}^{\frac{3}{4}}}/{n}$ in the above equation, but we can omit it as we only concern the smallest and largest order.
	By verifying the variance-expectation condition, we have
	\begin{equation}\label{eq:varexpect}
		\mathbb{E} [ \mathcal{E}^t_h(\widehat{Q}_h^t ) - \mathcal{E}^t_h(f_h^{\star} ) ]^2 \leqslant 16 H^2 \mathbb{E}[\mathcal{E}^t_h(\widehat{Q}_h^t ) - \mathcal{E}^t_h(f_h^{\star} ) ]\,,
	\end{equation}
	where $f_h^{\star} := \argmin_{f \in \mathcal{F}_{\tt DNN}} \mathcal{E}^t_h (f)$ and we use the fact  $\mathcal{E}^t_h(f )$ is $4H$-Lipschitz. 
Denote  the function space 	$\widehat{\mathcal{F}_{\tt DNN}} $ with the following function formulation for any $j \in [n]$
	\begin{equation*}
		\hat{g}_h^t := \left[ \widehat{Q}_h^t(s_h^{\tau_j}, a_h^{\tau_j}) - r_h(s_h^{\tau_j}, a_h^{\tau_j}) - {V}^t_{h+1}(s_{h+1}^{\tau_j}) \right]^2 - \left[ f_h^{\star}(s_h^{\tau_j}, a_h^{\tau_j}) - r_h(s_h^{\tau_j}, a_h^{\tau_j}) - {V}^t_{h+1}(s_{h+1}^{\tau_j}) \right]^2\,,
	\end{equation*}
	we have $P_n \hat{g}_h^t = \widehat{\mathcal{E}}^t_h(\widehat{Q}_h^t ) - \widehat{\mathcal{E}}^t_h(f_h^{\star} ) \leqslant 0$ due to $\widehat{Q}_h^t = \argmin_{f \in \mathcal{F}} \widehat{\mathcal{E}}^t_h(f)$.
	Then using $\mathbb{E} g^2 \leqslant H^2 Pg$, for any $g \in \widehat{\mathcal{F}_{\tt DNN}}$ by Eq.~\eqref{eq:varexpect}, according to Lemma~\ref{lemma:lrade}, the following inequality holds with probability at least $1-\delta$
	\begin{equation*}
		P \hat{g}_h^t \lesssim \frac{J}{H^2} R^* + \frac{(H^2J+H)\log(1/\delta)}{n}\,, \quad \forall J > 1\,,
	\end{equation*}
	where  which further implies
	\begin{equation*}
		\mathcal{E}^t_h(\widehat{Q}^t_h) - 	\min_{f \in \mathcal{F}_{\tt DNN}} \mathcal{E}^t_h(f) \lesssim  \frac{N \left[(\log N)^2 + \log n \right]}{n} + \frac{H \sqrt{N\left[(\log N)^2 + \log n \right]}}{n} + \frac{H^2 \log(1/\delta)}{n}\,.
	\end{equation*}
Finally, we conclude the proof.
\end{proof}

\begin{proof}[Proof of Theorem~\ref{thm:dnn}]
	Using the approximation error in $L^4(\mathcal{X})$ by Corollary~\ref{coro:tdgen}, the smoothness parameter satisfies $\alpha > d(1/p - 1/4)_+$. 
	By taking $\delta/2$ in Proposition~\ref{prop:genbes}, we have
	\begin{equation}\label{eqproofbes}
		\begin{split}
			\| \Gamma_h^t \|^2_{L^2{(\mathrm{d}\bar{\mu}_h^{\tilde{t}})}} & \lesssim \Big[ \mathcal{E}^t_h(\widehat{Q}^t_h) - 	\min_{f \in \mathcal{F}_{\tt DNN}} \mathcal{E}^t_h(f) \Big] + \inf_{f \in \mathcal{F}_{\tt DNN}} \| f - \mathbb{T}_h^{\star} Q^t_{h+1} \|^2_{L^4(\mathcal{X})} \quad [\mbox{using Corollary~\ref{coro:tdgen}}]\\
		& \lesssim N^{-\frac{2\alpha}{d}} + \frac{N \left[(\log N)^2 + \log \tilde{t} \right]}{\tilde{t}} + \frac{H \sqrt{N\left[(\log N)^2 + \log \tilde{t} \right]}}{\tilde{t}} + \frac{H^2 \log(2/\delta)}{\tilde{t}} \,,
		\end{split}
	\end{equation}
where in the second inequality, taking $\alpha > d(1/p - 1/4)_+$, the approximation error can be estimated by Lemma~\ref{lem:appbesov}
\begin{equation*}
	\inf_{f \in \mathcal{F}_{\tt DNN}} \| f - \mathbb{T}_h^{\star} Q^t_{h+1} \|^2_{L^4{(\mathcal{X})}} \lesssim N^{-2\alpha/d}\,.
\end{equation*}
Accordingly, the right hand side of \cref{eqproofbes} can be minimized by
taking $N \asymp \tilde{t}^{\frac{d}{2 \alpha + d}}$ up to $(\log \tilde{t})^3$-order in Eq.~\eqref{eq:lsmbbes} for choosing suitable $L,m, S,B$.
To make the architecture of deep RL independent of a variable $\tilde{t}$ (or $t$) during different episodes, here we directly choose $ N \asymp T^{\frac{d}{2 \alpha + d}} \log^3 T$, in this case, Eq.~\eqref{eqproofbes} can be formulated as
\begin{equation*}
\begin{split}
	\| \Gamma_h^t \|^2_{L^2{(\mathrm{d}\bar{\mu}_h^{\tilde{t}})}} & \lesssim H T^{-\frac{2\alpha}{2 \alpha +d}} \log^3 \tilde{t} + \frac{T^{\frac{d}{2\alpha+d}} \log^5 T}{\tilde{t}} + \frac{H^2 \log(2/\delta)}{\tilde{t}}  \,,
\end{split}    
\end{equation*}
which requires the depth $L$ and the width $m$ up to
	\begin{equation*}
	L \asymp \frac{d}{2\alpha +d} \log T, \quad  m \asymp \frac{d}{2\alpha +d} T^{\frac{d}{2\alpha +d}} \log T\,.
	\end{equation*}

Recall $\tilde{t} := \lceil \varrho t \rceil$, according to \cref{lem:termitd}, if $\alpha > d(1/p - 1/4)_+$, then for any $\delta \in (0,1)$, with probability at least $1 - \delta/2$, the ${\tt Term(i)}$ can be upper bounded by
	\begin{equation}\label{eq:termilocalbes}
		\begin{split}
			{\tt Term(i)}	
			& \lesssim \left( \frac{\epsilon}{A} \right)^{-\frac{K}{2}} H \sqrt{T} \sqrt{ \sum_{t=1}^T T^{-\frac{2\alpha}{2 \alpha +d}} \log^3 \varrho t + \frac{T^{\frac{d}{2\alpha+d}} \log^5 T}{\varrho t} + \frac{H^2 \log(2/\delta)}{\varrho t} } + H \sqrt{T} \\
			& \lesssim \left( \frac{\epsilon}{A} \right)^{-\frac{K}{2}} H \sqrt{T} \sqrt{H T^{\frac{d}{2 \alpha +d}} \log^3 T +\frac{1}{\varrho} \int_{1}^{T+1} \left(\frac{T^{\frac{d}{2\alpha+d}} \log^5 T}{t} + \frac{H^2 \log(2/\delta)}{t} \right) \mathrm{d} t} + H \sqrt{T} \\
		& \lesssim \left( \frac{\epsilon}{A} \right)^{-\frac{K}{2}} \frac{1}{\sqrt{\varrho}} \left( \sqrt{TH^3} \sqrt{T^{\frac{d}{2 \alpha + d}} \log^6 T} + H^2\sqrt{T} \sqrt{\log(2/\delta) \log T} \right) + H \sqrt{T} \\
		& \lesssim \left( \frac{\epsilon}{A} \right)^{-\frac{K}{2}} \frac{1}{\sqrt{\varrho}} \left( H^{\frac{3}{2}} T^{\frac{\alpha + d}{2\alpha + d}} \log^{3} T + H^2\sqrt{T} \sqrt{\log(2/\delta)} \log T \right) + H \sqrt{T} \,.
		\end{split}
	\end{equation}
	Then taking $\delta/2$ in the statistical error ${\tt Term(ii)}$ in \cref{lem:regret_decomp}, if $\alpha > d(1/p - 1/4)_+$, with probability at least $1-\delta$, we have
	\begin{equation*}
		\begin{split}
			\text{Regret}(T)  &\lesssim \left( \frac{\epsilon}{A} \right)^{-\frac{K}{2}} \frac{1}{\sqrt{\varrho}} \left(  H^{\frac{3}{2}} T^{\frac{\alpha + d}{2\alpha + d}} \log^{3} T  + H^2 \sqrt{T}  \sqrt{\log \Big( \frac{2}{\delta} \Big)} \log T \right)  + \epsilon HT + \sqrt{TH^3 \log \Big( \frac{4}{\delta} \Big)}\,.
		\end{split}
	\end{equation*}
Then taking
\begin{equation*}
    \epsilon = \mathcal{O}(H^{\frac{2}{K+2}} K^{\frac{2}{K+2}} A^{\frac{K}{K+2}} T^{-\frac{2\alpha}{ (2\alpha + d) (K+2) }})\,,
\end{equation*}
which implies
\begin{equation*}
        \text{Regret} \lesssim \widetilde{\mathcal{O}}(H^{\frac{H+4}{H+2}} K^{\frac{2}{K+2}} A^{\frac{K}{K+2}} T^{\frac{\alpha K + (\alpha + d)(K+2)}{(2\alpha + d)(K+2)}}) \,.
\end{equation*}
Finally we conclude the proof.

\end{proof}

\section{Proofs of regret bounds via two-layer neural networks}
\label{app:2nn}

In this section, we focus on generalization bounds under  the independent but non-identically distributed data setting in the Barron space, and it is useful to present estimates of our regret bound.

\begin{lemma}\label{lem:rade2nnpath}
	For two-layer ReLU neural networks with bounded $\ell_1$ path norm defined in Eq.~\eqref{eq2nndef} given the function class $\mathcal{F}_{\tt SNN}$ and $n$ independent but non-identically distributed data points $X = \{ \bm x_i \}_{i=1}^n \subseteq \mathcal{X}$, then we have
	\begin{equation*}
		\mathcal{R}_{n}(\mathcal{F}_{\tt SNN}) \leqslant 2B \sqrt{\frac{2\log (2d)}{n}}\,.
	\end{equation*}
\end{lemma}
\begin{proof}
	Here we directly focus on the $\ell_1$ path norm, which is different from  \cite[Theorem 3]{weinan2021barron}.
	Based on the definition of two-layer ReLU neural networks defined in Eq.~\eqref{eq2nndef}, denote $\widetilde{\bm w}_k := (\bm w_k^{\!\top},c_k)^{\!\top}$ and $\widetilde{\bm x} = (\bm x^{\!\top},1)^{\!\top}$ for simplicity, the empirical Rademacher complexity of $\mathcal{F}_{\tt SNN}$ under our setting can be upper bounded by
\begin{equation*}
	\begin{split}
		\widehat{\mathcal{R}}_{n}(\mathcal{F}_{\tt SNN}, X)	& = \frac{1}{n}\mathbb{E}_{ \bm \xi} \left[\sup _{f \in \mathcal{F}_{\tt SNN}} \frac{1}{m} \sum_{k=1}^{m} b_{k} \sum_{i=1}^{n} \xi_i   \sigma(\widetilde{\bm w}_k^{\!\top} \widetilde{\bm x})  \right] \\
		& \leqslant \mathbb{E}_{\bm \xi} \left[\sup _{f \in \mathcal{F}_{\tt SNN}} \frac{1}{m} \sum_{k=1}^{m}\left| b_{k}\right|\left\| \widetilde{\bm w}_{k}\right\|_{1} \frac{1}{n}\left|\sum_{i=1}^{n} \xi_i \sigma\left(\frac{ \widetilde{\bm w}_{i}}{\left\| \widetilde{\bm w}^{\top}_{i}\right\|_{1}} \widetilde{\bm x}_{i} \right)\right|\right] \\
		& \leq B \mathbb{E}_{ \bm \xi} \left[\sup _{\| \widetilde{\bm w} \|_{1} \leqslant 1} \frac{1}{n}\left|\sum_{i=1}^{n} \xi_i  \sigma(\widetilde{\bm w}^{\!\top} \widetilde{\bm x}_i) \right|\right] \\
		& \leqslant 2B \mathbb{E}_{\bm \xi} \left[\sup _{\| \widetilde{\bm w} \|_{1} \leqslant 1} \frac{1}{n} \sum_{i=1}^{n} \xi_i  \widetilde{\bm w}^{\!\top} \widetilde{\bm x}_i \right] \quad [\mbox{using symmetry of $\bm \xi$ and $1$-Lipschitz of ReLU}] \\
		& \leqslant 2B \mathbb{E}_{\bm \xi} \left\| \frac{1}{n} \sum_{i=1}^n \xi_i \widetilde{\bm x}_i \right\|_{\infty}\,, \quad [\mbox{using H\"{o}lder inequality}]
	\end{split}
\end{equation*}
where the first inequality holds by  the homogeneity of ReLU for any $\widetilde{\bm w} \in \mathbb{R}^d / \{ \bm 0 \} $.
Since the Massart's lemma is still valid under our independent but non-identically distributed data, $	\mathcal{R}_{n}(\mathcal{F}_{\tt SNN})$ can be further expressed by
\begin{equation*}
		\widehat{\mathcal{R}}_{n}(\mathcal{F}_{\tt SNN}, X) \leqslant 2B \mathbb{E}_{\bm \xi} \left\| \frac{1}{n} \sum_{i=1}^n \xi_i \widetilde{\bm x}_i \right\|_{\infty} \leqslant 2B \sqrt{2 \log (2 d) / n} \,,
\end{equation*}
where the last inequality holds by the maximum of $n$ sub-Gaussian random variables \cite{wainwright2019high} since Rademacher random variables are sub-Gaussian, and finally we conclude the proof.
\end{proof}

\begin{proof}[Proof of Theorem~\ref{thm:2nn}]
	Denote $X:=\{ ( s_h^{\tau_j}, a_h^{\tau_j}, s_{h+1}^{\tau_j} ) \}_{j = 1}^{\tilde{t}}$ for simplicity and notice that the function $[f(s_h, a_h) - r_h(s_h, a_h) - V^t_{h+1}(s_{h+1})]^2$ is $4H$-Lipschitz.
	Then according to Lemma~\ref{lem:expRn}, for any $\delta \in (0,1)$, the following result holds with probability at least $1-\delta/2$
\begin{equation}\label{eq:gen2nn}
	\begin{split}
		\widehat{\mathcal{E}}^t_h(f) - \mathcal{E}^t_h(f)  &\leqslant 2 \widehat{\mathcal{R}}_{t-1}(\mathcal{F}_{\tt SNN}, X) + 12H \sqrt{\frac{\log(4/\delta)}{2\tilde{t}}} \\
		& \leqslant 8 B \widetilde{R}H \sqrt{\frac{2 \log (2d)}{\tilde{t}}}  + 12H \sqrt{\frac{\log (4/\delta)}{2\tilde{t}}} \,,
	\end{split}
\end{equation}
where we use the empirical Rademacher complexity in Lemma~\ref{lem:rade2nnpath}.
Accordingly, by Lemma~\ref{lem:tdgen} and Eq.~\eqref{eq:gen2nn}, then with probability at least $1 - \delta/2$, we have 
	\begin{equation}\label{eqproofrc}
	\begin{split}
		  \| \Gamma_h^t \|^2_{L^2{(\mathrm{d}\mu_h^{\tilde{t}})}} & \leqslant \Big[ \mathcal{E}_h(\widehat{Q}^t_h) - 	\min_{f \in \mathcal{F}_{\tt SNN}} \mathcal{E}_h(f) \Big] + \inf_{f \in \mathcal{F}_{\tt SNN}} \| f - \mathbb{T}_h^{\star} Q^t_{h+1} \|^2_{L^2{(\mathrm{d}\bar{\mu}^{\tilde{t}}_h)}}  \\
		& \leqslant \mathcal{E}_h(\widehat{Q}^t_h) - 	\min_{f \in \mathcal{F}_{\tt SNN}} \mathcal{E}_h(f) + \frac{3 \| \mathbb{T}_h^{\star} Q^t_{h+1} \|_{\mathcal{P}}^2}{m}   \\
		& \leqslant 2 \sup_{f \in \mathcal{F}_{\tt SNN}}|\mathcal{E}^t_h(f) - \widehat{\mathcal{E}}^t_h(f)| + \frac{3\widetilde{R}^2}{m}  \quad [\mbox{using Assumption~\ref{ass:barron}}] \\
		& \leqslant 16B \widetilde{R}H \sqrt{\frac{2 \log (2d)}{\tilde{t}}}  + 24H \sqrt{\frac{\log (4/\delta)}{2\tilde{t}}} + \frac{3\widetilde{R}^2}{m}  \,, \quad [\mbox{using Eq.~\eqref{eq:gen2nn}}] 
	\end{split}
\end{equation}
where the second inequality uses the approximation result for two-layer ReLU neural networks and the Barron space in \cite[Theorem 4]{weinan2021barron}.
Accordingly, by Lemma~\ref{lem:termitd}, for any $\delta \in (0,1)$, the ${\tt Term(i)}$ in the regret decomposition can be upper bounded with probability at least $1 - \delta/2$
\begin{equation}\label{eq:termi}
	\begin{split}
		{\tt Term(i)}
		& \lesssim \left( \frac{\epsilon}{A} \right)^{-\frac{K}{2}} H\sqrt{T} \sqrt{\sum_{t=1}^T \left(B H^2 (\varrho t)^{-\frac{1}{2}} \sqrt{\log d}+ H (\varrho t)^{-\frac{1}{2}} \sqrt{\log \frac{4}{\delta}} + \frac{H^2}{m}  \right) } + H\sqrt{T}  \\
		& \lesssim \left( \frac{\epsilon}{A} \right)^{-\frac{K}{2}} H\sqrt{T}  \sqrt{\frac{1}{\sqrt{\varrho}} \int_{1}^{T+1} \left(B H^2 t^{-\frac{1}{2}} \sqrt{\log d}+ H t^{-\frac{1}{2}} \sqrt{\log \frac{4}{\delta}}  \right) \mathrm{d} t + \frac{H^2T}{m}} + H\sqrt{T}  \\
		& \lesssim  \left( \frac{\epsilon}{A} \right)^{-\frac{K}{2}} \frac{1}{\sqrt{\varrho}} \left[ T^{\frac{3}{4}} H^2 B (\log d)^{\frac{1}{4}}  + T^{\frac{3}{4}} H^{\frac{3}{2}} \log^{\frac{1}{4}} \Big( \frac{4}{\delta} \Big) \right] + \left( \frac{\epsilon}{A} \right)^{-\frac{K}{2}}\frac{H^2T}{\sqrt{m}} + H\sqrt{T} \,.
	\end{split}
\end{equation}
where we use $\widetilde{R} \asymp H$ and $\int_{1}^T (t-1)^{-1/2} \mathrm{d} t = \mathcal{O}(\sqrt{T})$.

Accordingly, taking $\delta/2$ in the statistical error ${\tt Term(ii)}$ in \cref{lem:regret_decomp}, then with the probability at least $1-\delta$, the total regret can be upper bounded by
\begin{equation*}
		\begin{split}
			\text{Regret}(T)  & \lesssim \left( \frac{\epsilon}{A} \right)^{-\frac{K}{2}} \!\! \left( \frac{H^2 T^{\frac{3}{4}}}{\sqrt{\varrho}} \left[ B(\log d)^{\frac{1}{4}} \!+\! \log^{\frac{1}{4}} \Big( \frac{4}{\delta} \Big)  \right] \!+\! \frac{H^2T}{\sqrt{m}} \right)  \!+\! \epsilon HT \!+\! \sqrt{TH^3 \log \Big( \frac{4}{\delta} \Big)}\,.
		\end{split}
	\end{equation*}
	Taking $m = \Omega(\sqrt{T})$ and $\epsilon = \mathcal{O} \left(H^{\frac{2}{K+2}} T^{-\frac{1}{2(K+2)}} \right)$, the regret bound can be further represented as
\begin{equation*}
    	\text{Regret}(T) \lesssim \widetilde{\mathcal{O}}(	H^{\frac{K+4}{K+2}} T^{\frac{2K+3}{2K+4}} ) \,,
\end{equation*}
which concludes the proof.
\end{proof}

\end{document}